\documentclass{article}

\usepackage{times}
\usepackage{graphicx}

\usepackage{subfigure}
\newlength{\subfigwidth}
\newlength{\subfigcolsep}
\usepackage{url}
\usepackage{bm}
\usepackage{amsfonts}
\usepackage{amsthm}
\usepackage{amsmath}
\usepackage{mathtools}
\usepackage{amssymb}

\usepackage{natbib}

\usepackage{algorithm}
\usepackage{algorithmic}

\usepackage[nohyperref,accepted]{icml2016} 

\DeclareMathOperator*{\argmin}{\mathop{\rm arg~min}}

\newtheorem{theorem}{Theorem}
\newtheorem{lemma}[theorem]{Lemma}
\newtheorem{fact}[theorem]{Fact}

\newcommand{\Prob}{\mathbb{P}}
\newcommand{\Expect}{\mathbb{E}}
\newcommand{\Indicator}{{\mathbf{1}}}
\newcommand{\Ind}{\Indicator}

\newcommand{\Regret}{{R}}

\newcommand{\reg}[2]{{r}_{{#1},{#2}}}
\newcommand{\regf}[2]{{\hat{r}}_{{#1},{#2}}}

\newcommand{\Bernoulli}{\mathrm{Bernoulli}}
\newcommand{\muij}[2]{{\mu_{{#1},{#2}}}}
\newcommand{\nuij}[2]{{\nu_{{#1},{#2}}}}

\newcommand{\nuijd}[2]{{\nu_{{#1},{#2}}'}}
\newcommand{\muijd}[2]{{\mu_{{#1},{#2}}'}}
\newcommand{\hatmut}[2]{{\hat{\mu}}_{{#1},{#2}}(t)}
\newcommand{\hatmun}[3]{{\hat{\mu}}_{{#1},{#2}}^{#3}}
\newcommand{\Xij}{{\hat{X}_{i,j}}}
\newcommand{\xij}[2]{{x_{#1,#2}}}
\newcommand{\myX}[2]{{\hat{X}_{{#1},{#2}}}}

\newcommand{\EA}{\mathcal{A}}
\newcommand{\EB}{\mathcal{B}}

\newcommand{\ED}{\mathcal{D}}
\newcommand{\EE}{\mathcal{E}}

\newcommand{\EXt}[2]{\mathcal{X}_{#1,#2}(t)}
\newcommand{\EXdt}[2]{\mathcal{X}_{#1,#2}'(t)}
\newcommand{\EYt}[2]{\mathcal{Y}_{#1,#2}(t)}
\newcommand{\EZt}[1]{\mathcal{Z}_{#1}(t)}
\newcommand{\EXct}[2]{\mathcal{X}_{#1,#2}^c(t)}

\newcommand{\EZct}[1]{\mathcal{Z}_{#1}^c(t)}
\newcommand{\is}{{i^*}}
\newcommand{\ist}{{\hat{i}^*(t)}}

\newcommand{\nonwinners}{{[K] \setminus \winners}}
\newcommand{\nn}{\nonumber\\}

\newcommand{\rd}{\mathrm{d}}
\newcommand{\e}{\mathrm{e}}
\newcommand{\Natural}{\mathbb{N}}

\newcommand{\Real}{\mathbb{R}}

\newcommand{\myalpha}{\alpha}

\newcommand{\mcondor}{\mathcal{M}_{\mathrm{Cond}}}
\newcommand{\mcop}{\mathcal{M}_{\mathrm{Cop}}}

\newcommand{\Li}[1]{{L_{#1}}}
\newcommand{\Lif}[1]{{\hat{L}_{#1}}}
\newcommand{\Lifo}[1]{{\hat{L}^{(#1)}}}
\newcommand{\Lid}[1]{{L_{#1}'}}
\newcommand{\Lone}{L_{1}}
\newcommand{\Ltwo}{L_{2}}

\newcommand{\arms}{[K]}
\newcommand{\winners}{[C]}
\newcommand{\winnersf}{\hat{\mathcal{C}}_{\mathrm{cop}}}

\newcommand{\mO}{{\mathcal{S}}}
\newcommand{\SetO}[1]{{S_{#1}}}
\newcommand{\SetOf}[1]{{\hat{S}_{#1}}}
\newcommand{\SetH}[1]{{I_{#1}}}
\newcommand{\SetHf}[1]{{\hat{I}_{#1}}}

\newcommand{\SubPowSetO}[2]{{\mathcal{S}_{#1}^{#2}}}
\newcommand{\SubPowSetOf}[2]{{\hat{\mathcal{S}}_{#1}^{#2}}}
\newcommand{\SubPowSetORem}[3]{{\mathcal{S}_{#1}^{\setminus #3, #2}}}
\newcommand{\SubPowSetORemf}[3]{{\hat{\mathcal{S}}_{#1}^{\setminus #3, #2}}}
\newcommand{\SubPowSetH}[2]{{\mathcal{I}_{#1}^{#2}}}
\newcommand{\SubPowSetHf}[2]{{\hat{\mathcal{I}}_{#1}^{#2}}}
\newcommand{\lo}{{O}}
\newcommand{\so}{{o}}
\newcommand{\KL}{{d_{\mathrm{KL}}}}

\newcommand{\Nt}[2]{{N_{#1,#2}(t)}}
\newcommand{\myN}[3]{{N_{#1,#2}(#3)}}
\newcommand{\Nbeat}[2]{{N_{{#1}>{#2}}(t)}}
\newcommand{\NT}[2]{{N_{#1,#2}(T)}}
\newcommand{\NTst}[2]{{N_{#1,#2}(t)}}

\newcommand{\nSum}[1]{{N_{#1}^{\mathrm{sum}}}}

\newcommand{\sumdiv}[1]{{e_{#1}^{\mathrm{Sum}}(T)}}
\newcommand{\lbdelta}{\delta}
\newcommand{\lbepsilon}{{\epsilon_1}}
\newcommand{\lbepsilontwo}{{\epsilon_2}}

\newcommand{\Expectp}{{\Expect'}}
\newcommand{\Probp}{{\Prob'}}
\newcommand{\Pij}[2]{{P_{#1,#2}}}
\newcommand{\hKL}{\mathrm{\widehat{KL}}}
\newcommand{\hKLij}[2]{{\mathrm{\widehat{KL}}_{#1,#2}}}
\newcommand{\nij}[2]{{n_{#1,#2}}}
\newcommand{\Mat}{M}
\newcommand{\Matp}{{\Mat'}}
\newcommand{\SetIS}{{\mathcal{P}_{IS}}}
\newcommand{\pairs}{{\mathcal{P}_{\mathrm{all}}}}
\newcommand{\SKP}{{\mathcal{P}_{i\neq j}}}

\newcommand{\cost}[1]{{c_{#1}}}
\newcommand{\normvar}[1]{{e_{#1}}}
\newcommand{\normvarstar}[1]{{e_{#1}^*}}
\newcommand{\normvardash}[1]{{e_{#1}'}}
\newcommand{\normvardashtwo}[1]{{e_{#1}''}}

\newcommand{\optvar}[2]{{q_{#1,#2}}}
\newcommand{\optvard}[2]{{q_{#1,#2}'}}
\newcommand{\optvaro}[2]{{q_{#1,#2}^*}}

\newcommand{\admq}[2]{\mathcal{R}_{#1}(#2)}
\newcommand{\admqs}[1]{\mathcal{R}_{#1}^{+1}}
\newcommand{\admqzero}[2]{\mathcal{R}_{#1}^{+1}(#2)}
\newcommand{\admqrelax}[2]{\mathcal{R}_{#1}^{\mathrm{E}}(#2)}
\newcommand{\optset}[2]{\mathcal{R}_{#1}^*(#2)}
\newcommand{\optsetzero}[2]{\mathcal{R}_{#1}^{+1,*}(#2)}
\newcommand{\optsetij}[2]{R_{#1,#2}^*}
\newcommand{\optsets}[1]{\mathcal{R}_{#1}^{+1,*}}
\newcommand{\optsetrelax}[2]{\mathcal{R}_{#1}^{\mathrm{E}*}(#2)}
\newcommand{\optsetrelaxij}[2]{R_{#1,#2}^{\mathrm{E}*}}
\newcommand{\optcone}[1]{C_{#1}^*}
\newcommand{\optconezero}[1]{C_{#1}^{+1,*}}
\newcommand{\optconerelax}[1]{C_{#1}^{\mathrm{E}*}}

\newcommand{\com}{\,,}
\newcommand{\per}{\,.}

\newcommand{\sensedelta}{\delta}

\newcommand{\senseepsilon}{\epsilon}

\newcommand{\algalpha}{\alpha}
\newcommand{\algbeta}{\beta}

\newcommand{\lllT}{{F(T)}}

\newcommand{\closure}{\mathrm{cl}}
\newcommand{\cred}[1]{\textcolor{red}{#1}}

\icmltitlerunning{Copeland Dueling Bandit Problem: Regret Lower Bound, Optimal Algorithm, and Computationally Efficient Algorithm}

\begin{document} 

\twocolumn[
\icmltitle{Copeland Dueling Bandit Problem: Regret Lower Bound,\\ Optimal Algorithm, and Computationally Efficient Algorithm}

\icmlauthor{Junpei Komiyama}{junpei@komiyama.info}
\icmladdress{The University of Tokyo, Japan}
\icmlauthor{Junya Honda}{honda@stat.t.u-tokyo.ac.jp}
\icmladdress{The University of Tokyo, Japan}
\icmlauthor{Hiroshi Nakagawa}{nakagawa@dl.itc.u-tokyo.ac.jp}
\icmladdress{The University of Tokyo, Japan}

\icmlkeywords{multi-armed bandit problem, dueling bandit problem, online learning, regret, learning theor}

\vskip 0.3in
]

\begin{abstract} 
We study the $K$-armed dueling bandit problem, a variation of the standard stochastic bandit problem where the feedback is limited to relative comparisons of a pair of arms. The hardness of recommending Copeland winners, the arms that beat the greatest number of other arms, is characterized by deriving an asymptotic regret bound. We propose Copeland Winners Relative Minimum Empirical Divergence (CW-RMED) and derive an asymptotically optimal regret bound for it. However, it is not known whether the algorithm can be efficiently computed or not. To address this issue, we devise an efficient version (ECW-RMED) and derive its asymptotic regret bound. Experimental comparisons of dueling bandit algorithms show that ECW-RMED significantly outperforms existing ones.
\end{abstract} 

%\allowdisplaybreaks[1]

\section{Introduction}
\vspace{-0.2em}

A multi-armed bandit problem is a crystallized instance of a sequential decision-making problem in an uncertain environment, and it can model many real-world scenarios. This problem involves conceptual entities called arms. At each round, the forecaster draws one of the $K$ arms and receives a corresponding reward feedback.
The aim of the forecaster is to maximize the cumulative reward over rounds, which is achieved by running an algorithm that balances the exploration (acquisition of information) and the exploitation (utilization of information). In evaluating the performance of a bandit algorithm, a metric called regret, which measures how much the algorithm explores, is widely used. 

While it is desirable to obtain rewards as direct feedback from an arm, in a number of practical cases such direct feedback is not available. In this paper, we consider a version of the standard stochastic bandit problem called the $K$-armed dueling bandit problem \citep{DBLP:conf/colt/YueBKJ09}, in which the forecaster receives relative feedback, which specifies which of the two arms is preferred.
Although the original motivation of the dueling bandit problem arose in the field of information retrieval, learning under relative feedback is universal to many fields, such as recommender systems \citep{Gemmis09preferencelearning}, graphical design \citep{DBLP:conf/sca/BrochuBF10}, and natural language processing \citep{DBLP:conf/acl/ZaidanC11}, which involve explicit or implicit feedback provided by humans. 

In the standard bandit problem, the best arm is naturally defined as the one with the largest expected reward. 
However, if the feedback is restricted to the results of pairwise comparisons, there are several possible ways to define the best arm. Following the literature on the dueling bandit problem, we call the best arm the winner. 
When there exists an arm that beats (i.e., preferred in expectation) all the other arms, it is natural to define it as the winner; this notion is called a Condorcet winner. Unfortunately, the Condorcet winner does not always exist.
Still, we can define an extended notion of the Condorcet winner that always exists as follows. Let the Copeland winners be the arms that beat the greatest number of other arms. In this paper, we study the difficulty of finding the Copeland winners from pairwise feedback.

\vspace{-0.2em}
\subsection{Related work}
\vspace{-0.2em}

Early algorithms for solving the dueling bandit problem, such as Interleaved Filter \cite{DBLP:journals/jcss/YueBKJ12} and Beat the Mean Bandit \cite{DBLP:conf/icml/YueJ11}, require the arms to be totally ordered.

\citet{DBLP:conf/icml/UrvoyCFN13} considered a large class of sequential learning problems that includes the dueling bandit problem and introduced the notion of Condorcet, Copeland, and Borda dueling bandit problems. 
Several algorithms, such as Relative Upper Confidence Bound (RUCB) \cite{DBLP:conf/icml/ZoghiWMR14}, and Relative Minimum Empirical Divergence (RMED) \cite{DBLP:conf/colt/KomiyamaHKN15}, have since been proposed that effectively solve the Condorcet dueling bandit problem. The assumption on the Condorcet winner partly relaxes the assumption of the total order because it admits circular preferences that involve non-winners. 

However, as the Condorcet winner does not always exist, the result of running one of these algorithms is unpredictable if it is applied to an instance without a Condorcet winner. Consequently, the practical applicability of Condorcet dueling bandit algorithms is limited. Some papers have discussed the problem of preference elicitation without a Condorcet winner \cite{DBLP:conf/aistats/JamiesonKDN15,zoghicopeland} and have motivated studies of more general dueling bandit problems. Unlike the Condorcet winner, the Borda and Copeland winners always exist. Note that there are also other notions of winners, such as the von Neumann winner \cite{DBLP:conf/colt/DudikHSSZ15} or Random walk winner \cite{DBLP:journals/jair/AltmanT08} together with their corresponding dueling bandit problems. Among them, we will consider the Copeland dueling bandit problem. Unlike the Borda or Random Walk winners, the Copeland winners are compatible with the Condorcet winner; if the Condorcet winner exists, it is also the Copeland winner. An algorithm for finding the Copeland winner (i) covers the application range of the Condorcet winner and (ii) can find arms that beat other arms the most, even if the Condorcet winner does not exist. 

Another line of study is on the partial monitoring problem \cite{DBLP:journals/mor/BartokFPRS14}. 
The partial monitoring is general enough to cover the multi-armed bandit. Some classes of dueling bandit problems, such as utility-based ones \cite{DBLP:journals/corr/GajaneU15}, can also be formalized as a partial monitoring. However, it is unknown as to whether the Copeland dueling bandit problem can be effectively represented as a partial monitoring or not. Moreover, existing algorithms for partial monitoring, such as Bayes-update Partial Monitoring (BPM) \cite{efficientpm} or Partial Monitoring Deterministic Minimum Empirical Divergence (PM-DMED) \cite{komiyama15pm}, are not very scalable to the number of actions.

\textbf{Existing results on the Copeland dueling bandit problem:}
The difficulty of the dueling bandit problem lies in that there are $O(K^2)$ pairs.
There are some algorithms, such as Sensitivity Analysis of VAriables for Generic Exploration (SAVAGE) \cite{DBLP:conf/icml/UrvoyCFN13}, Preference-based Racing (PBR) \cite{DBLP:conf/icml/Busa-FeketeSCWH13}, and Rank Elicitation (RankEI) \cite{DBLP:conf/aaai/Busa-FeketeSH14}, that can deal with general classes of problems that entail solving Copeland dueling bandit problems. The price to pay for such generality is performance: all three algorithms have $O(K^2 \log{T})$ regret because they naively compare all pairs $O(\log{T})$ times.

The recently proposed Copeland Confidence Bound (CCB) \cite{zoghicopeland} exploits the structure of the Copeland dueling bandit problem and is relatively efficient. It has an asymptotic regret of $\lo(\frac{K(C+L_1+1)}{\Delta^2}\log{T})$ (Theorem 3 in \citealt{zoghicopeland}), where $C$ is the number of Copeland winners, $L_1$ is the number of arms that beats the Copeland winner, and $\Delta$ is related to how hard it is to determine whether each arm $i$ beats arm $j$ or not. In this paper, we further push our understanding of the dueling bandit problem by deriving an asymptotically optimal regret bound. The optimal bound states that (i) the dependency on $C$ can be completely removed; (ii) the dependency on $L_1$ is necessary for some cases but unnecessary for typical cases explained later, and (iii) the dependency on $\Delta$ can be relaxed by introducing a divergence-based bound.
In an information retrieval example, the optimal bound improves the one of the CCB by several orders of magnitude (Table \ref{tbl_bounds}).

\begin{table}[b!]
\vspace{-1em}
\caption{Comparison of leading logarithmic constants of regret bounds on the Microsoft Learning to Rank dataset. CW-RMED and ECW-RMED are the algorithms proposed in this paper. 
The values are averaged over $10^4$ randomly generated submatrices of size $16\times16$. Details of the dataset are presented in Section \ref{sec_experiment}. The bound of CCB is about $1000$ times looser than the optimal.}
\begin{center}
{\renewcommand\arraystretch{2.2}
  \begin{tabular}{|c|c|c|} \hline
   \shortstack{\\Optimal: \\ CW-RMED} & ECW-RMED & \shortstack{\\CCB \\\cite{zoghicopeland}} \\\hline\hline
   $6.7 \times 10^2$ & $7.3 \times 10^2$ & $8.8 \times 10^5$ \\\hline
  \end{tabular}
}
\end{center}
\label{tbl_bounds}
\end{table}

\textbf{Contributions:}
The main contributions of this paper are summarized in the following four aspects:
First, we derive an asymptotic regret lower bound (Section \ref{sec_lower}). The lower bound is based on the minimum amount of exploration for identifying a Copeland winner. 
Second, we propose the Copeland Winners Relative Minimum Empirical Divergence (CW-RMED) algorithm. CW-RMED is the first algorithm whose performance asymptotically matches the regret lower bound (Section \ref{subsec_alg_cw}). Unfortunately, a naive implementation of CW-RMED is computationally prohibitive.
Third, we propose Efficient Copeland Winners RMED (ECW-RMED), another algorithm that addresses the above computational issue (Section \ref{subsec_alg_ecw}). An efficient way to implement it is proposed. Moreover, we show that the regret of ECW-RMED is very close to optimal. 
Finally, we implemented ECW-RMED and compared its performance with those of existing algorithms (Section \ref{sec_experiment}). ECW-RMED significantly outperformed the state-of-the-art algorithms on many datasets. In a ranker evaluation example, its regret was smaller than one third of those of the others.

\vspace{-0.2em}
\section{Problem Setup}
\vspace{-0.2em}

The $K$-armed dueling bandit problem involves $K$ arms that are indexed as $[K] := \{1,2,\dots,K\}$. 
Let $\Mat \in \Real^{K \times K}$ be a preference matrix whose $ij$ entry $\muij{i}{j}$ corresponds to the probability that arm $i$ is preferred to arm $j$.
At each round $t=1,2,\dots,T$, the forecaster draws a pair of arms $p(t) = (l(t), m(t)) \in [K]^2$ and, receives relative feedback $\myX{l(t)}{m(t)}(t) \sim \Bernoulli(\muij{l(t)}{m(t)})$ that indicates which of $(l(t), m(t))$ is preferred. We say arm $i$ beats arm $j$ if $\muij{i}{j}>1/2$. By definition, $\muij{i}{j} = 1 - \muij{j}{i}$ holds for any $i,j \in [K]$ and $\muij{i}{i} = 1/2$. 
Throughout this paper, we assume $\muij{i}{j} \ne 1/2$ for $i \ne j$. Let $\SKP := \{(i,j) : i,j \in [K], i>j\}$ and $\pairs := \{(i,j) : i,j \in [K], i \ge j\}$. A comparison of pair $(i,j)$ is identified with that of pair $(j,i)$.

Let $\Nt{i}{j}$ be the number of comparisons of pair $(i,j)$ and $\hatmut{i}{j}$ be the empirical estimate of $\muij{i}{j}$ at round $t$. For $j \neq i$, let $\Nbeat{i}{j}$ be the number of times $i$ is preferred over $j$. Accordingly, $\hatmut{i}{j} = \Nbeat{i}{j}/\Nt{i}{j}$, where we set $0/0 = 1/2$ here. 
%In building these statistics, we treat pairs without taking their order into consideration. 

Let the superiors of arm $i$ be $\SetO{i} := \{j:j \in \arms, \muij{i}{j} <1/2\}$, that is, the set of arms that beat arm $i$. Let $\Li{i} := |\SetO{i}|$ and $C = |\{i\in\arms:\Li{i}=\min_j \Li{j}\}|$. 
Without loss of generality, we can assume $\Li{1} = \Li{2}=\dots =\Li{C}\le \dots \le \Li{K}$. Of course, algorithms should not exploit this ordering. Arms $\winners$ are called Copeland winners. Note that the Copeland winners always exist, but are not necessarily unique. Let the inferiors of arm $i$ be $\SetH{i} := \{j:j \in \arms, \muij{i}{j} >1/2\}$. Assuming that $\muij{i}{j} \ne 1/2$ for $i \ne j$, each arm $j$ is either a superior or an inferior of arm $i$. 
When $\Li{1} = 0$, the Copeland winner is unique and also called a Condorcet winner.

We define the regret per round\footnote{The constant factor of this definition is different from the one defined in \citet{zoghicopeland}. Our result can be compared with that of \citet{zoghicopeland} simply by multiplying a constant.} is $\reg{i}{j} := (\Li{i} + \Li{j} - 2 \Li{1}) / (2(K-1)) \le 1$ when the pair $(i,j)$ is compared and the regret as $ \Regret(T) := \sum_{t \in [T]} \reg{l(t)}{m(t)}$.
The regret increases at each round unless both $l(t)$ and $m(t)$ are Copeland winners. 
This definition is reasonable because we have defined the goodness of an arm by the number of arms that $i$ beats (Copeland number) and are interested in drawing the best arms. The choice of $l(t) = m(t)$ is possible, but yields no useful information since, by definition, $\muij{i}{i} = 1/2$ for any arm $i$.

Note that, we can also consider other definitions of regret; the analysis in this paper is relied on the facts that regret per round $\reg{i}{j}$ is (i) finite, (ii) determined by the Copeland numbers, (iii) and equal to zero if $i$ and $j$ are Copeland winners. For example, we can consider a regret such that 
$\reg{i}{j} = 0$ if $i,j \in \winners$ and $1$ otherwise, and easily modify our result in accordance with that definition.

\vspace{-0.2em}
\section{Regret Lower Bound}
\label{sec_lower}
\vspace{-0.2em}

In this section, we derive an asymptotic regret lower bound when $T \rightarrow \infty$. 
In the context of the standard multi-armed bandit problem, \citet{LaiRobbins1985} derived the regret lower bound of strongly consistent algorithms; intuitively, a strongly consistent algorithm is ``uniformly good'' in the sense that it works well with any set of model parameters. We extend this result to the Copeland dueling bandit problem.

We first define notions that are important in characterizing the regret lower bound: the subsets of the power set of the superiors and the inferiors with a fixed size. Let $\SubPowSetO{i}{m}:=\{S \in 2^{\SetO{i}}: |S|=m\}$, $\SubPowSetH{i}{m}:=\{I \in 2^{\SetH{i}}: |I|=m\}$, and $\SubPowSetORem{i}{m}{j}:=\{S \in 2^{\SetO{i} \setminus \{j\}}: |S|=m\}$.
Moreover, let $\mcop$ be a set of all preference matrices of size $K \times K$. 
A Copeland dueling bandit algorithm is strongly consistent if it satisfies $\Expect[\Regret(T)] = \so(T^a)$ for any $a>0$ given any preference matrix $\Mat \in \mcop$.
Essentially, a strongly consistent algorithm needs to find one of the Copeland winners with a high confidence level. 
To make sure that arm $i^*$ is a Copeland winner, we need to simultaneously find (i) an upper-bound $\Li{i^*}$ of a Copeland winner $i^*$ and (ii) a lower-bound $\Li{j}$ of the other arms. The minimum amount of exploration in Copeland dueling bandit is characterized in this way. The following lemma formalizes the aforementioned statement.
\begin{lemma} (Lower bound on the number of draws)
\label{lem_lbnumarm}
Let $\KL(p, q) := p \log{p/q} + (1-p) \log{(1-p)/(1-q)}$ be the Kullback-Leibler (KL) divergence between two Bernoulli distributions with parameters $p,q$. For any strongly consistent algorithm, the following inequality holds for at least one $i_1 \in \winners$:
\begin{multline}
 \forall_{i_2 \neq i_1}\,\forall{l \in \{\max\{0,\Lone-1\},\dots,\Ltwo\} }
\\
\forall{I \in \SubPowSetH{i_1}{l+1-\Li{i_1}}}\,\forall{S \in \SubPowSetORem{i_2}{\max\{0,\Li{i_2}-l-\Ind\{i_2 \in I\}\}}{i_1}}
\\
\sum_{(i,j) \in \SetIS} \KL(\muij{i}{j}, 1/2) \Expect\left[\NT{i}{j}\right] \geq (1-\so(1)) \log{T},
\label{ineq_lbnumarm}
\end{multline}
\vspace{-0.5em}
where 
\begin{align*}
 \SetIS & = \SetIS(i_1, i_2, l, I,S) \\
  &:= \{(i_1, j) : j \in I\} \cup \{(i_2,j) : j \in S\}.
\end{align*}
\end{lemma}%
Intuitively, Lemma \ref{lem_lbnumarm} can be interpreted as follows: for each round $t$, consistency requires an algorithm to identify one of the Copeland winners $i_1$ with confidence level $1/t$. For some $i_2,l,I,$ and $S$, if the preferences among the  pairs in $\SetIS(i_1, i_2, l, I,S)$ are inverted, then arm $i_2 \ne i_1$ has $\Li{i_2} \le l$ and $\Li{i_1} \ge l+1$, which implies that $i_1$ is not a Copeland winner. We need to limit all such risks for all possible $i_2,l,I,S$. Each risk is calculated in accordance with the large deviation principle \cite{coverthomas2nd} as $\sim \exp{(-\sum_{\SetIS} \KL(\muij{i}{j}, 1/2) \NTst{i}{j})}$, and the algorithm must continue comparing pairs in $\SetIS$ until $\sum_{\SetIS} \KL(\muij{i}{j}, 1/2) \NTst{i}{j} \sim \log{t}$ in order to lower the risk to $\exp{(- \log{t})} = 1/t$ for each $\SetIS$. 

The proof of Lemma \ref{lem_lbnumarm} is in Appendix \ref{sec_prooflower}. Note that all proofs are in Supplementary Material. The technique used in the proof extends the one of \citet{LaiRobbins1985} for the standard bandit problem in two aspects: (i) in the standard bandit problem, each arm is associated with a single distribution, whereas  in the Copeland dueling bandit problem each arm is related to $K-1$ distributions (i.e., comparison with other arms). Therefore, not all of the pairs are required to be drawn, and we need a sophisticated analysis to determine the set of conditions that consistency requires. Moreover, (ii) the Copeland winner is not necessarily unique; there can be several ties with the maximum Copeland number. We show that consistency requires an algorithm to find at least one of the Copeland winners, but it does not need to find all of them. 

Next, we derive the asymptotic regret lower bound, which is the minimum amount of regret such that inequality \eqref{ineq_lbnumarm} is satisfied. Let $\{\nuij{i}{j}\}$ be a $K\times K$ preference matrix, and let the superiors and the inferiors under the preference matrix $\{\nuij{i}{j}\}$ be $\SetOf{i} = \SetOf{i}(\{\nuij{i}{j}\}) := \{j\in\arms:\nuij{i}{j}< 1/2\}$ and $\SetHf{i} = \SetHf{i}(\{\nuij{i}{j}\}) := \{j\in\arms:\nuij{i}{j}> 1/2\}$.
Moreover, let the number of the superiors be $\Lif{i} = \Lif{i}(\nuij{i}{j}) := |\SetOf{i}(\nuij{i}{j})|$, and the $a$-th smallest element among $\{\Lif{i}\}_{i \in \arms}$ be $\Lifo{a} = \Lifo{a}(\{\nuij{i}{j}\})$; let the Copeland winner be $\winnersf = \winnersf(\{\nuij{i}{j}\}) := \{i : \Lif{i} = \Lifo{1}(\{\nuij{i}{j}\}) \} \subset [K]$. $\SubPowSetOf{i}{m}(\{\nuij{i}{j}\})$, $\SubPowSetHf{i}{m}(\{\nuij{i}{j}\})$, and $\SubPowSetORemf{i}{m}{j}(\{\nuij{i}{j}\})$ are defined in the same way.
For $i_1 \in \winnersf$, let 
\begin{align}
\lefteqn{
\hspace{-0.5em}\admq{i_1}{\hspace{-0.1em}\{\nuij{i}{j}\}} \hspace{-0.2em} := \hspace{-0.2em} \Biggl\{\hspace{-0.2em} \{\optvar{i}{j}\}_{i>j} \hspace{-0.2em} \in\hspace{-0.1em}  [0,1/\KL(\nuij{i}{j}, 1/2)]^{K(K-1)/2} : 
} \nn
&\hspace{2.0em}\forall_{i_2 \neq i_1}\,\forall{l \in \{\max\{0,\Lifo{1}-1\},\dots,\Lifo{2}\} }
 \nn
&\hspace{2.0em}\forall{I \in \SubPowSetHf{i_1}{(l+1-\Lifo{1})}}\,\forall{S \in \SubPowSetORemf{i_2}{\max\{0, \Lif{i_2}-l-\Ind\{i_2 \in I\}\}}{i_1}} \nn
&\hspace{2.0em}\sum_{(i,j) \in \SetIS} \optvar{i}{j} \KL(\nuij{i}{j}, 1/2) \geq 1
 \Biggr\}.\label{optcond}
\end{align}
Note that $\admq{i_1}{\{\nuij{i}{j}\}}$ is non-empty because it includes a trivial solution $\optvar{i}{j} = 1/\KL(\nuij{i}{j},1/2)$ for each $(i,j) \in \SKP$. 
Moreover, let $\regf{i}{j}(\{\nuij{i}{j}\}) := (\Lif{i}+\Lif{j} - 2 \Lifo{1})/(2(K-1))$ be the regret per draw with $\{\nuij{i}{j}\}$ and 
\begin{equation*}
 \optcone{i_1}(\{\nuij{i}{j}\}) := \inf_{\{\optvar{i}{j}\}_{i>j}\in\admq{i_1}{\{\nuij{i}{j}\}}} \sum_{(i,j) \in \SKP} \regf{i}{j} \optvar{i}{j} \com
\end{equation*}
and let the (possibly non-unique) set of optimal solutions be
\begin{align*}
\lefteqn{
 \optset{i_1}{\{\nuij{i}{j}\}} :=\biggl\{ \{\optvar{i}{j}\}_{i>j}  \in \admq{i_1}{\{\nuij{i}{j}\}} :
}\\
&\hspace{7em}  \sum_{(i,j) \in \SKP} \regf{i}{j} \optvar{i}{j} = \optcone{i_1}(\{\nuij{i}{j}\})
\biggr\}\per %\label{def_optset}
\end{align*}

The value $\optcone{i_1}(\{\muij{i}{j}\}) \log{T}$ is the possible minimum regret for exploration to make sure that the arm $i_1$ is in $\winners$. Using Lemma \ref{lem_lbnumarm} yields the following regret lower bound.
\begin{theorem} 
The regret of a strongly consistent algorithm is lower bounded as:
\begin{equation*}
\Expect[\Regret(T)]\ge
\min_{i_1 \in \winners} \optcone{i_1}(\{\muij{i}{j}\}) \log{T}-\so(\log T).
\end{equation*}
\label{thm_regretlower}
\end{theorem}%
\vspace{-2em}
The proof of Theorem \ref{thm_regretlower} is in Appendix \ref{sec_prooflower}.

\subsection{Comparison with the Consistency in Condorcet dueling bandits}
\label{subsec_comconsistency}
A dueling bandit algorithm is strongly consistent in the sense of Condorcet if it has subpolynomial regret for any $\Mat \in \mcondor$, where $\mcondor$ is the set of preference matrices in which the Condorcet winner (i.e., the Copeland winner $i_1$ of $\Li{i_1}=0$) exists \cite{DBLP:conf/colt/KomiyamaHKN15}. 
Although the definitions of the regret in the two dueling bandit problems are slightly different, they are the same in that drawing pairs that include non-Copeland winners increases regret, and thus a subpolynomial regret in the sense of the Condorcet dueling bandit problem is consistent with the one of the Copeland dueling bandit problem. Therefore, a strongly consistent Copeland dueling bandit algorithm is also strongly consistent in the sense of the Condorcet dueling bandit problem since $\mcop \supset \mcondor$. The converse is not necessarily true: when we run a Condorcet dueling bandit algorithm with a preference matrix without a Condorcet winner, it can fail to identify a Copeland winner. 

An example in which the two consistencies make a difference is in Table \ref{tbl_cyclic}. RMED2FH \cite{DBLP:conf/colt/KomiyamaHKN15}, an optimal algorithm for solving the Condorcet dueling bandit problem, may not be consistent in the sense of Copeland; RMED2FH draws pairs $(2,3), (2,4)$, and $(3,4)$ to prove that each of $2,3$, and $4$ is beaten by another arm, which implies that these arms are non-Condorcet. 
However, in the sense of Copeland, an algorithm must make sure that the superior of arm $1$ is smaller than those of the other arms, and thus, it needs to compare arm $1$ with the others for sufficiently many times.

\begin{table}[b!]
\vspace{-1em}
\caption{A preference matrix of size $4\times4$. The $ij$-th element is $\muij{i}{j}$. The Copeland (Condorcet) winner is arm $1$.}
\begin{center}
  \begin{tabular}{|l|l|l|l|l|}  \hline
      & 1 & 2 & 3 & 4 \\ \hline  
    1 & 0.5 & 0.6 & 0.6 & 0.6 \\ \hline  
    2 & 0.4 & 0.5 & 0.9 & 0.1 \\ \hline  
    3 & 0.4 & 0.1 & 0.5 & 0.9 \\ \hline  
    4 & 0.4 & 0.9 & 0.1 & 0.5 \\ \hline  
  \end{tabular}
\end{center}
\label{tbl_cyclic}
\end{table}

\section{Algorithms}

In this section, we first introduce the CW-RMED algorithm, which is inspired by the DMED algorithm \cite{HondaDMED} for solving the multi-armed bandit problem. We then derive an asymptotically optimal regret bound for CW-RMED. 
However, to the best of our knowledge, it is not known whether an optimization in the subroutine can be efficiently computed or not. To address this issue, we devise another algorithm called ECW-RMED, which is computationally efficient and has a regret bound that is close to optimal.

\begin{algorithm}[t!]
 \caption{CW-RMED and ECW-RMED Algorithms}
 \label{alg_rmedbase}
\begin{algorithmic}[1]
   \STATE \textbf{Input:} $K$ arms, $\algalpha > 0$, $\algbeta>0$.
   \STATE $L_C, L_R \leftarrow \SKP, L_N \leftarrow \emptyset$.
   \WHILE{$t \le T$} %\label{line_alg_loop}
     \STATE Draw all pairs such that $(i,j) \in \SKP$ if $\Nt{i}{j} < \algalpha \sqrt{\log{t}}$ or $|\hatmut{i}{j} - 1/2| < \algbeta/\log{\log{t}}$. $t \leftarrow t+1$ for each draw. %\label{line_sqrtexploration}
     \FOR{$p(t)=(l(t),m(t)) \in L_C$ in an arbitrarily fixed order} 
       \STATE Draw arm pair $p(t)$. %\label{line_listdraw}
       \STATE $L_{NC} \leftarrow \emptyset$.
       \IF{
\begin{equation}
\{\Nt{i}{j}/\log t\}_{i \ne j} \in\admq{\ist}{\{\hatmut{i}{j}\}}
  \label{ineq_check}
\end{equation}
\\ for some $\ist \in \winnersf(\hatmut{i}{j})$} %\label{line_check}
         \STATE Put $(\ist, \ist)$ into $L_{NC}$. %\label{line_exploitone}
       \ELSE
         \STATE Compute some\phantom{wwwwwwwwwwwwwwwwwwwwwww}
\hspace{-1em}$\ist = \begin{cases}
  \argmin_{i_1 \in \winnersf} \optcone{i_1}(\{\hatmut{i}{j}\})  
 & \text{\hspace{-0.5em}(CW)} \\
  \argmin_{i_1 \in \winnersf} \optconerelax{i_1}(\{\hatmut{i}{j}\}) & \text{\hspace{-0.5em}(ECW)} \\
        \end{cases}
        $
\hspace{-0.5em}$\{\optvaro{i}{j}\} \in  \begin{cases}
          \optset{\ist}{\hatmut{i}{j}} & \text{\hspace{5.3em}(CW)} \\
         \optsetrelax{\ist}{\hatmut{i}{j}} & \text{\hspace{5.3em}(ECW)} \\
        \end{cases}
        $
(ties are broken arbitrarily) and put all pairs $(i,j) \in \SKP$ such that $\optvaro{i}{j} > \Nt{i}{j}/\log t$ into $L_{NC}$. %label{line_listup}
         \STATE Put $(\ist, \ist)$ into $L_{NC}$. %\label{line_exploittwo}
       \ENDIF
       \STATE $L_R \leftarrow L_R \setminus \{p(t)\}$.
       \STATE $L_N \leftarrow L_N \cup (i,j)$ (without a duplicate) for any $(i,j) \in L_{NC} \cap (\SKP \setminus L_R)$.
       \STATE $t \leftarrow t+1$.
     \ENDFOR
     \STATE $L_C, L_R \leftarrow L_N$, $L_N \leftarrow \emptyset$.
   \ENDWHILE
\end{algorithmic}
\end{algorithm}

\vspace{-0.2em}
\subsection{CW-RMED}
\vspace{-0.2em}
\label{subsec_alg_cw}

Algorithm \ref{alg_rmedbase} is CW-RMED. 
At the beginning of each round $t = 1,2,\dots,T$, if there exists a pair $(i,j)\in\SKP$ that is not drawn $O(\sqrt{\log t})$ times or $\hatmut{i}{j}$ is very close to $1/2$, it immediately draws that pair. 
Otherwise, it enters the loop that sequentially draws each pair in $L_C$.
After drawing each pair, it checks whether the current observation is sufficient or not. If the observation is enough to identify some $\ist$ as a Copeland winner, it exploits by adding $(\ist, \ist)$ into $L_{NC}$, the candidates of the pairs that will be drawn in the next loop. Otherwise, it draws the pairs with the number of observations below the minimum requirement for identifying $\ist$ as a winner with high confidence.
Note that it considers a pair of the same arm $(i,i)$ and pair of different arms $(i,j), i \ne j$, separately. Since a comparison with itself yields no information, drawing $(i,i)$ is purely for exploitation. 

The following theorem, whose proof is in Appendix \ref{sec_mainproof}, states that the regret of CW-RMED is asymptotically optimal when we view the parameters of the preference matrix $\{\muij{i}{j}\}$ as constants. 
Therefore, it performs as well as any other strongly consistent algorithm for sufficiently large $T$.
\begin{theorem}
Assume that $\argmin_{i_1 \in \winners} \optcone{i_1}(\{\muij{i}{j}\})$ and $\optset{i_1}{\{\muij{i}{j}\}}$ for each $i_1 \in \winners$ are unique.
For any $\algalpha > 0$, $\algbeta>0$, the regret of CW-RMED is bounded as:
\begin{equation*}
\Expect[\Regret(T)]\le \min_{i_1 \in \winners} \optcone{i_1}(\{\muij{i}{j}\}) \log T +\so(\log T)\per
\end{equation*}
\label{thm_opt}
\end{theorem}%
\vspace{-1em}

\vspace{-0.8em}
\subsubsection{Computation of an optimal solution}
\vspace{-0.2em}

Here, we discuss the computational aspects of CW-RMED.
Checking \eqref{ineq_check} is relatively easy since we can sort $\{\optvar{i}{j} \KL(\nuij{i}{j}, 1/2)\}$ for each $(i_1,j) \in \SetH{i_1}$ or $(i_2,j) \in \SetO{i_2}$, and the constraint that matters is the top-$c$ smallest of them for each size-$c$ subset. 

The difficult part is the computation of $\{\optvaro{i}{j}\} \in \optset{i_1}{\{\hatmut{i}{j}\}}$ for each $i_1$, which can be formulated as a linear programming (LP).
In the case of this paper the number of constraints of the LP is exponential in $K$
and a naive use of an LP solver is sometimes very slow.
It is well known that even if there are exponentially many constraints
an LP can be solved by using the ellipsoid method \cite{KHACHIYAN198053}
in a polynomial time if there exists a polynomial-time oracle that (i) checks whether a point $\{q_{ij}\}$ is feasible or not and (ii) returns a hyperplane such that $\{q_{ij}\}$ and the feasible region are separated if $\{q_{ij}\}$ is infeasible.
Such an oracle is easily constructed based on the sorting described above, and thus
$\{\optvaro{i}{j}\} \in \optset{i_1}{\{\hatmut{i}{j}\}}$ can be computed in a polynomial
time.
Although the ellipsoid method is practically very slow,
a practical combinatorial algorithm is often derived later for many problems that are solvable by the ellipsoid method (see, e.g., \citealt{Korte2007}, Chapters 1--4 and 12). Thus the authors think that $\optset{i_1}{\{\hatmut{i}{j}\}}$ can be computed practically.
Still, in this paper, we consider a suboptimal solution because it runs not only in polynomial time but also in time almost the same as that of sorting, as described in Section \ref{subsec_alg_ecw}.

\vspace{-0.2em}
\subsection{ECW-RMED}
\vspace{-0.2em}
\label{subsec_alg_ecw}

In this section, we propose ECW-RMED (Algorithm \ref{alg_rmedbase}).
The difference between CW-RMED (Section \ref{subsec_alg_cw}) and ECW-RMED is the amount of exploration. For a candidate of Copeland winners $i_1$, it tries to make sure that neither $\Li{i_1} \ge \min_{i} \Li{i} + 1$ nor $\Li{i_2} \le \min_{i} \Li{i} - 1$ for any $i_2 \ne i_1$ occurs, which implies that $i_1$ is a Copeland winner. Namely, for $i_1 \in \winnersf$, let 
\begin{align}
\lefteqn{
\hspace{-3em}\admqrelax{i_1}{\{\nuij{i}{j}\}} \hspace{-0.2em} := \hspace{-0.2em} \Biggl\{\hspace{-0.2em} \{\optvar{i}{j}\}_{i>j} \hspace{-0.2em} \in\hspace{-0.1em}  [0,1/\KL(\nuij{i}{j}, 1/2)]^{K(K-1)/2} : 
} \nn
&\hspace{2.0em}\cred{\forall{j \in \SetHf{i_1}}\,\optvar{i_1}{j} = 1/\KL(\nuij{i_1}{j}, 1/2),} \label{lwins} \\
&\hspace{2.0em}\cred{\forall_{i_2 \neq i_1} \forall{S \in \SubPowSetORemf{i_2}{\Lif{i_2}-\Lif{i_1}+1}{i_1}}} \nn
&\hspace{6.0em}\cred{\sum_{j \in S} \optvar{j}{i_2} \KL(\nuij{j}{i_2}, 1/2) \geq 1
 \Biggr\}.} \label{lloses}
\end{align}
Note that the red lines are the differences from $\admq{i_1}{\cdot}$.
Moreover, let 
\begin{equation*}
 \optconerelax{i_1}(\{\nuij{i}{j}\}) := \inf_{\{\optvar{i}{j}\}\in\admq{i_1}{\{\nuij{i}{j}\}}} \sum_{(i,j) \in \SKP} \regf{i}{j} \optvar{i}{j} \com
\end{equation*}
and let the (possibly non-unique) set of optimal solutions be
\begin{align*}
\lefteqn{
 \optsetrelax{i_1}{\{\nuij{i}{j}\}} :=\biggl\{ \{\optvar{i}{j}\} \in \admqrelax{i_1}{\{\nuij{i}{j}\}}:
}\nn
&\hspace{7em} \sum_{(i,j) \in \SKP} \regf{i}{j} \optvar{i}{j} = \optcone{i_1}(\{\nuij{i}{j}\})
\biggr\}\per %\label{def_optset_relax}
\end{align*}%
The following theorem, whose proof is in Appendix \ref{sec_mainproof}, bounds the regret of ECW-RMED.
\begin{theorem}
Assume that $\argmin_{i_1 \in \winners} \optconerelax{i}(\{\muij{i}{j}\})$ and $\optsetrelax{i_1}{\{\muij{i}{j}\}}$ for each $i_1 \in \winners$ are unique.
For any $\algalpha > 0$, $\algbeta>0$, the regret of ECW-RMED is bounded as:
\begin{equation*}
\Expect[\Regret(T)]\le \min_{i_1 \in \winners} \optconerelax{i_1}(\{\muij{i}{j}\}) \log T +\so(\log T)\per
\end{equation*}
\label{thm_relax}
\end{theorem}%
A quantitative discussion on the regret bounds of CW/ECW-RMED is found in Appendix \ref{sec_quantcompregret}.

\vspace{-0.2em}
\subsubsection{Efficient computation of ECW-RMED}
\vspace{-0.2em}
\label{subsec_eccompute}
 
In this section, we show an efficient method of finding $\{\optvar{i}{j}\}_{i>j} \in\optsetrelax{i_1}{\{\hatmut{i}{j}\}}$ for $i_1 \in \winnersf{(\{\hatmut{i}{j}\})}$. 
Since the inequality \eqref{lloses} is disjoint for each $i_2 \neq i_1$, solving it for each $i_2$ suffices. Let $\mO := \SetOf{i_2} \setminus \{i_1\}$, $k := |\mO|-(\Lif{i_2}-\Lif{i_1}+1)$. Moreover, let $\cost{j} := \regf{j}{i_2}(\{\hatmut{i}{j}\})/\KL(\hatmut{j}{i_2}, 1/2) \ge 0$ and $\normvar{j} := \optvar{j}{i_2} \KL(\hatmut{j}{i_2}, 1/2) \ge 0$. Accordingly, the regret minimization under \eqref{lloses} is reduced to the following linear optimization problem:
\begin{flalign}
\phantom{www}\text{minimize\hspace{1em}} & \sum_{j \in \mO} \cost{j} \normvar{j}&&\nn
\phantom{www}\text{subject to\hspace{1em}} & \forall_{S \subset \mO: |S|=|\mO|-k}\,\sum_{j \in S} \,\normvar{j} \geq 1.&&
\label{ineq_oopt}
\end{flalign}
Here, $\cost{j} \ge 0$ can be considered as a cost, and \eqref{ineq_oopt} is a cost minimization problem. In the following discussion we assume $|\mO|>k>0$; otherwise the optimization problem is trivial. The following theorem, whose proof is in Appendix \ref{sec_effcomp}, states that an optimal solution of the problem is computed efficiently.
\begin{theorem}
Let $\sigma_1, \sigma_2, \dots , \sigma_{|\mO|} \in \mO$ be a permutation of $\mO$ such that $\cost{\sigma_1} \le \cost{\sigma_2} \le \dots \le \cost{\sigma_{|\mO|}}$. There exists $h>k$ such that at least one optimal solution $\{\normvarstar{j}\}$ of \eqref{ineq_oopt} satisfies
\begin{multline}
\normvarstar{\sigma_1}=\normvarstar{\sigma_2}=\dots=\normvarstar{\sigma_h} = 1/(h-k),\\
\normvarstar{\sigma_{h+1}}=\normvarstar{\sigma_{h+2}}=\dots \normvarstar{\sigma_{|\mO|}}=0.
\label{equalblock}
\end{multline}
\label{thm_hsol}
\end{theorem}%
\vspace{-1em}
Since we only have $|\mO|-k \le K$ candidates of $h$ in \eqref{equalblock}, an optimal solution can be found by checking each of them.

\subsection{Relation between CW-RMED and ECW-RMED}

The following theorem, whose proof is in Appendix \ref{sec_proof_relaxation}, relates the optimal regret bound and the one of ECW-RMED.
\begin{theorem} {\rm (Optimality of ECW-RMED)}
The following inequality always holds:
\begin{equation}
 \optconerelax{i_1}(\{\muij{i}{j}\}) \ge \optcone{i_1}(\{\muij{i}{j}\}).
\label{optimality_always}
\end{equation}
Moreover, if $C \ge 2$, the following equality holds:
\begin{equation}
 \optconerelax{i_1}(\{\muij{i}{j}\}) = \optcone{i_1}(\{\muij{i}{j}\}).
\label{optimality_multi}
\end{equation}
\label{thm_relaxation}
\end{theorem}
Inequality \eqref{optimality_always} states that the leading logarithmic constant of  the bound on CW-RMED is always as good as that of ECW-RMED, which is natural since CW-RMED is asymptotically optimal as stated in Theorem \ref{thm_relax}. Still, \eqref{optimality_multi} states that ECW-RMED has exactly the same constant when the Copeland winners are not unique. 

\begin{figure*}[t!]
\begin{center}
  \setlength{\subfigwidth}{.33\linewidth}
  \addtolength{\subfigwidth}{-.33\subfigcolsep}
  \begin{minipage}[t]{\subfigwidth}
  \centering
 \subfigure[MSLR ($K=16$, Condorcet)]{
 \includegraphics[scale=0.381]{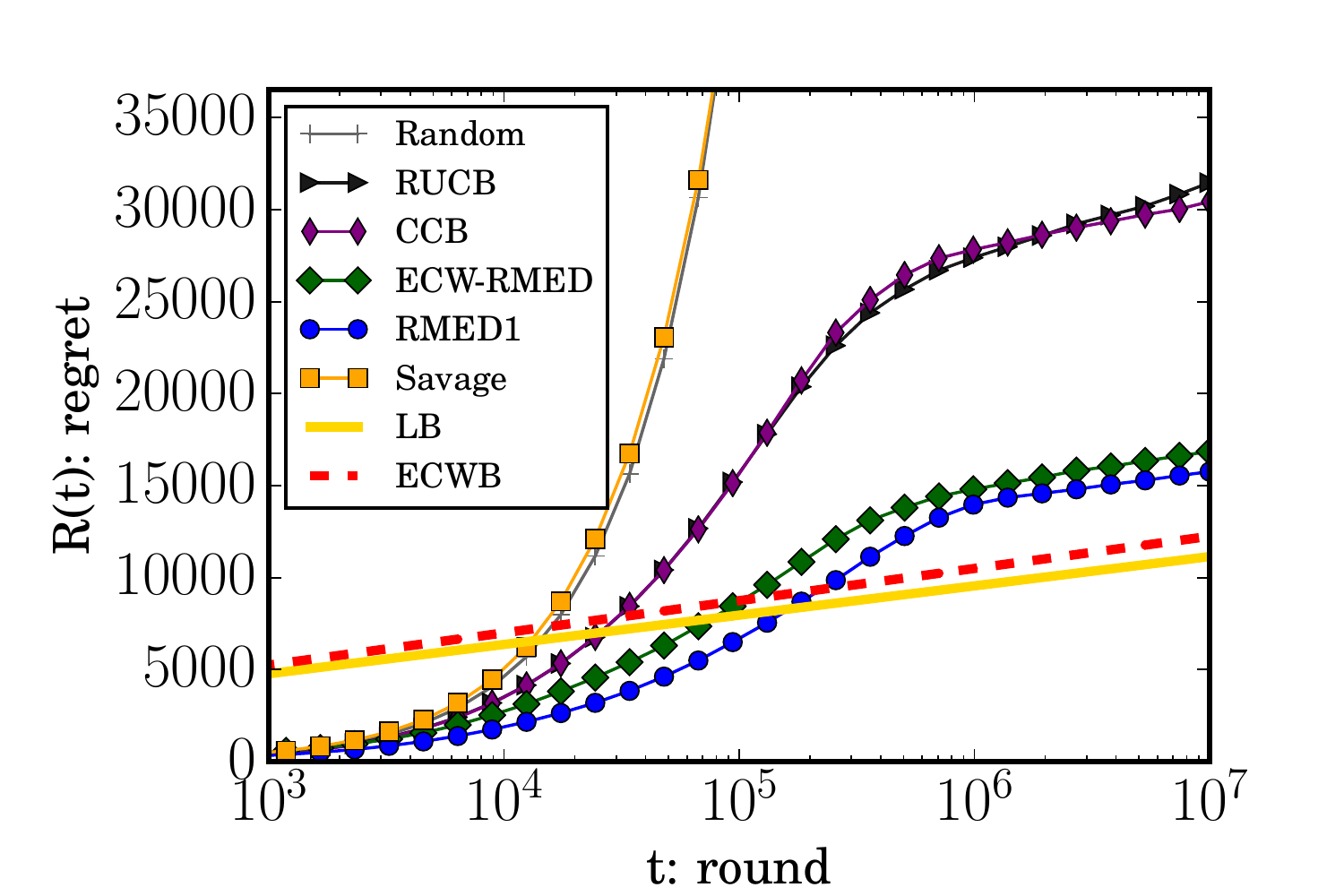}
 \label{fig_mslrcond}
 }
 \vspace{-0.501em}
 \end{minipage}\hfill
  \begin{minipage}[t]{\subfigwidth}
  \centering
 \subfigure[MSLR ($K=16$, non-Condorcet)]{\includegraphics[scale=0.381]{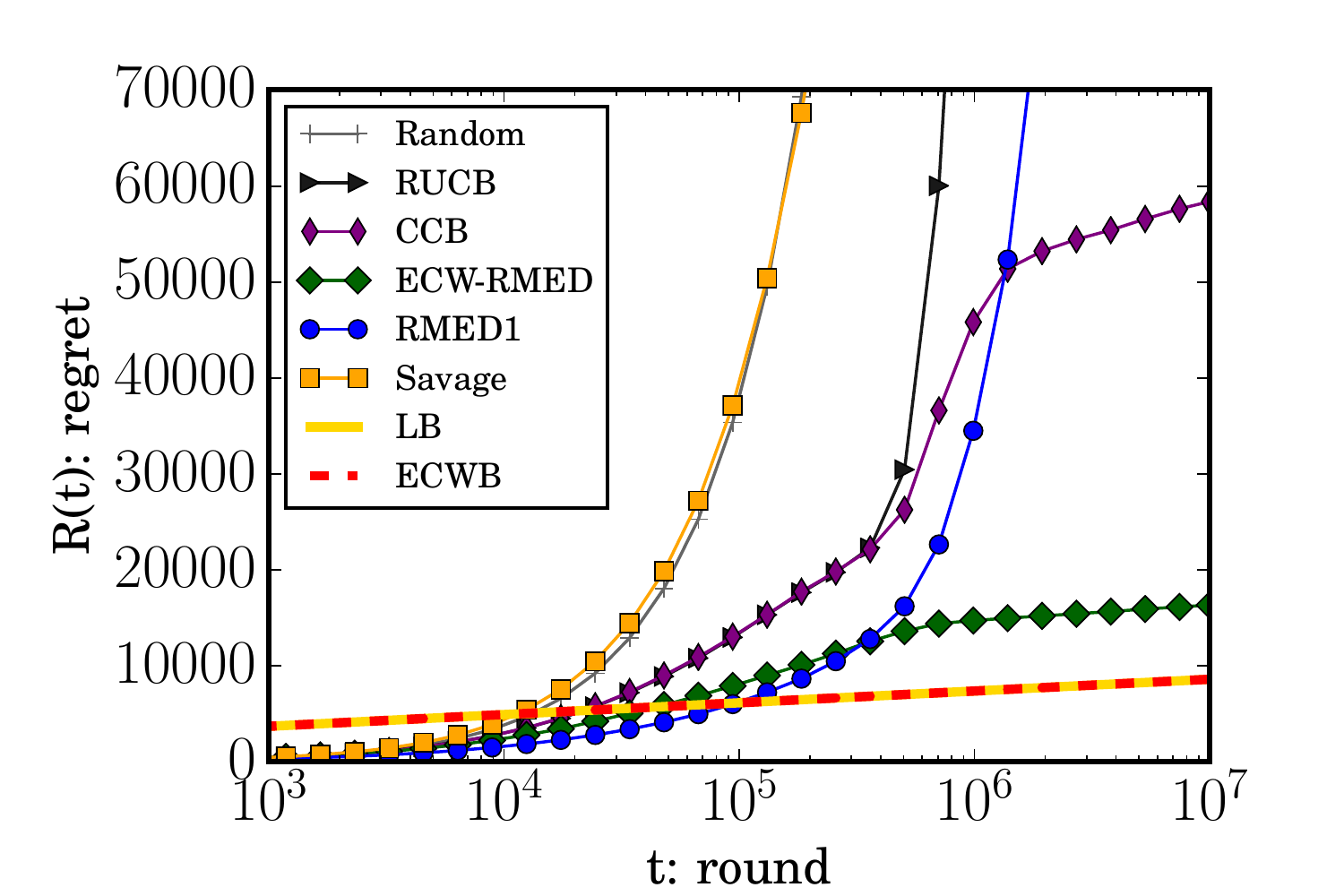}
 \label{fig_mslrnoncond}
 }
 \vspace{-0.501em}
  \end{minipage}\hfill
  \begin{minipage}[t]{\subfigwidth}
  \centering
 \subfigure[MSLR ($K=64$, non-Condorcet). LB was not computed because it is computationally too expensive for this value of $K$.]{
 \includegraphics[scale=0.381]{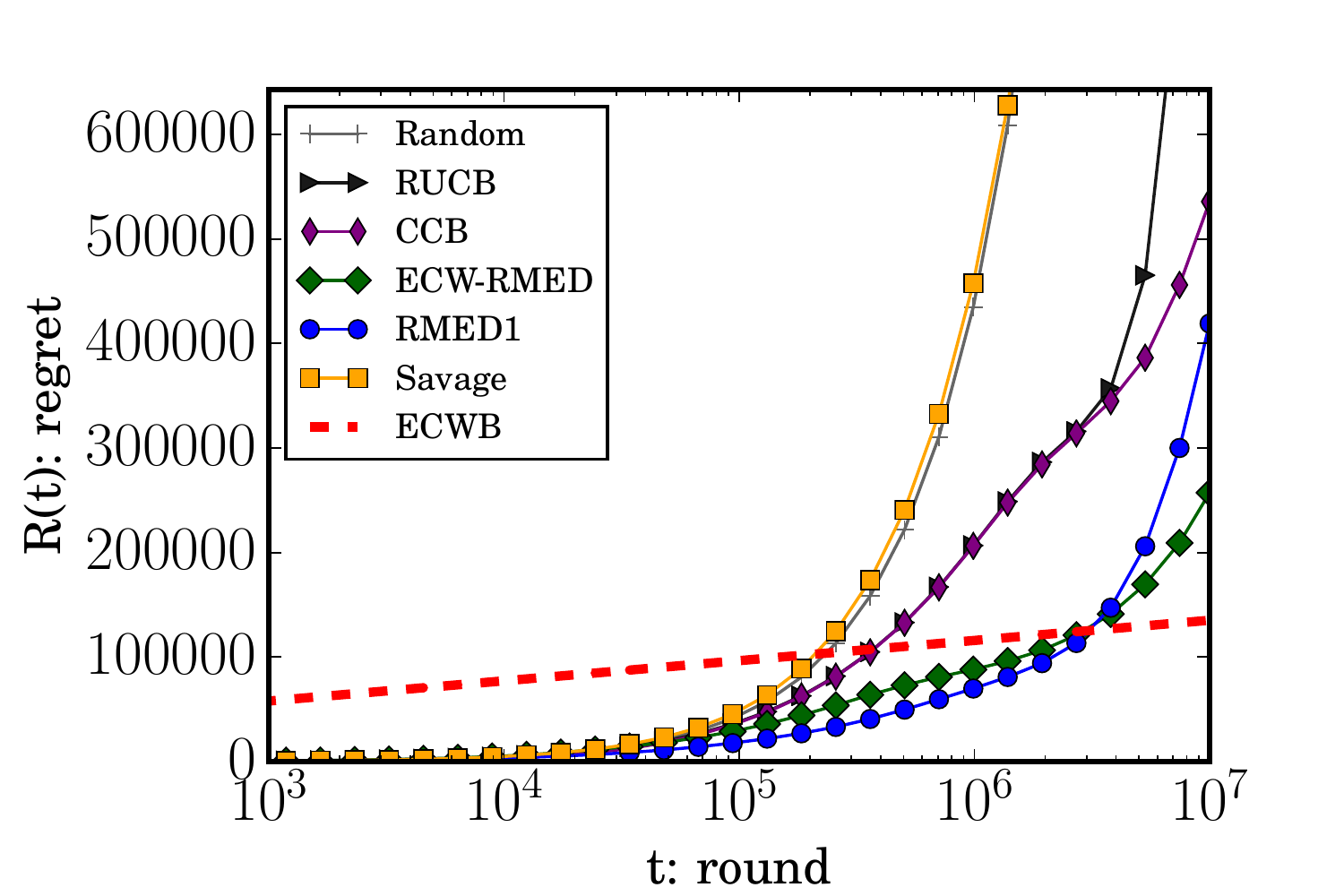}
 \label{fig_mslrlarge}
 }
 \vspace{-0.501em}
 \end{minipage}\hfill
  \begin{minipage}[t]{\subfigwidth}
  \centering
 \subfigure[Sushi]{\includegraphics[scale=0.381]{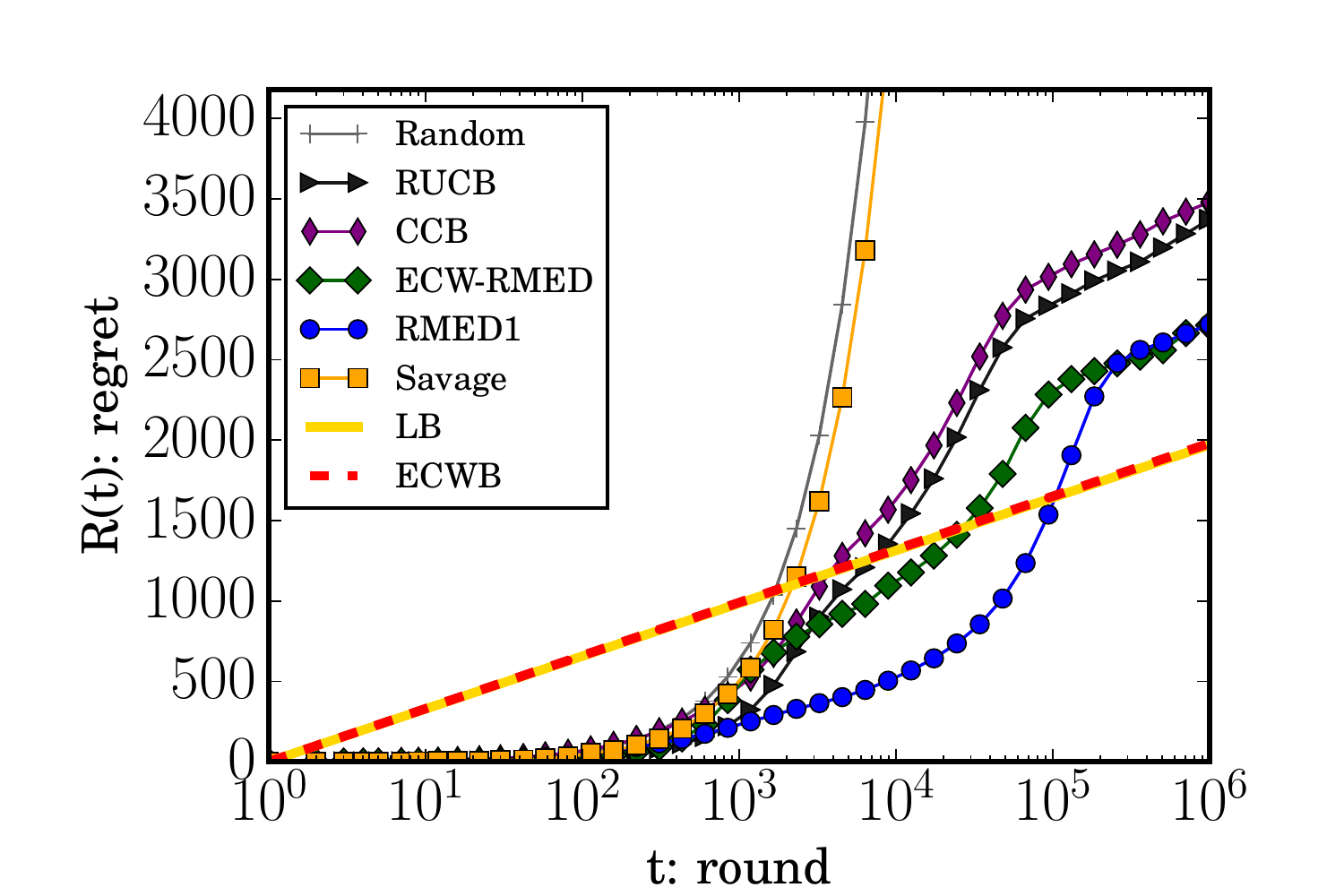}
 }
 \vspace{-0.501em}
  \end{minipage}\hfill
  \begin{minipage}[t]{\subfigwidth}
  \centering
 \subfigure[Gap; note that LB is very small (almost overlapping the bottom line).]{
 \includegraphics[scale=0.381]{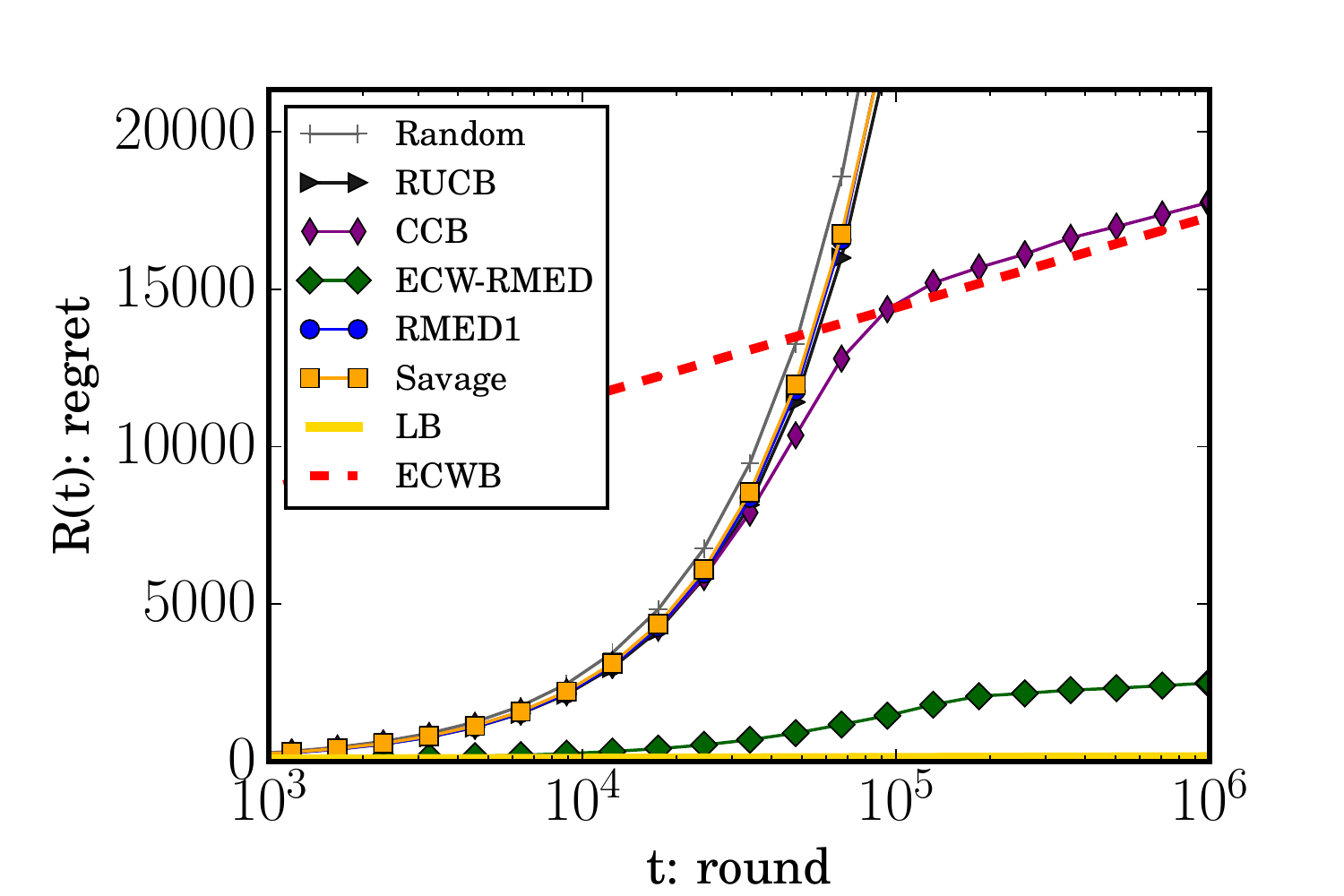}
 }
 \vspace{-0.501em}
 \end{minipage}\hfill
  \begin{minipage}[t]{\subfigwidth}
  \centering
 \subfigure[MultiSol]{\includegraphics[scale=0.381]{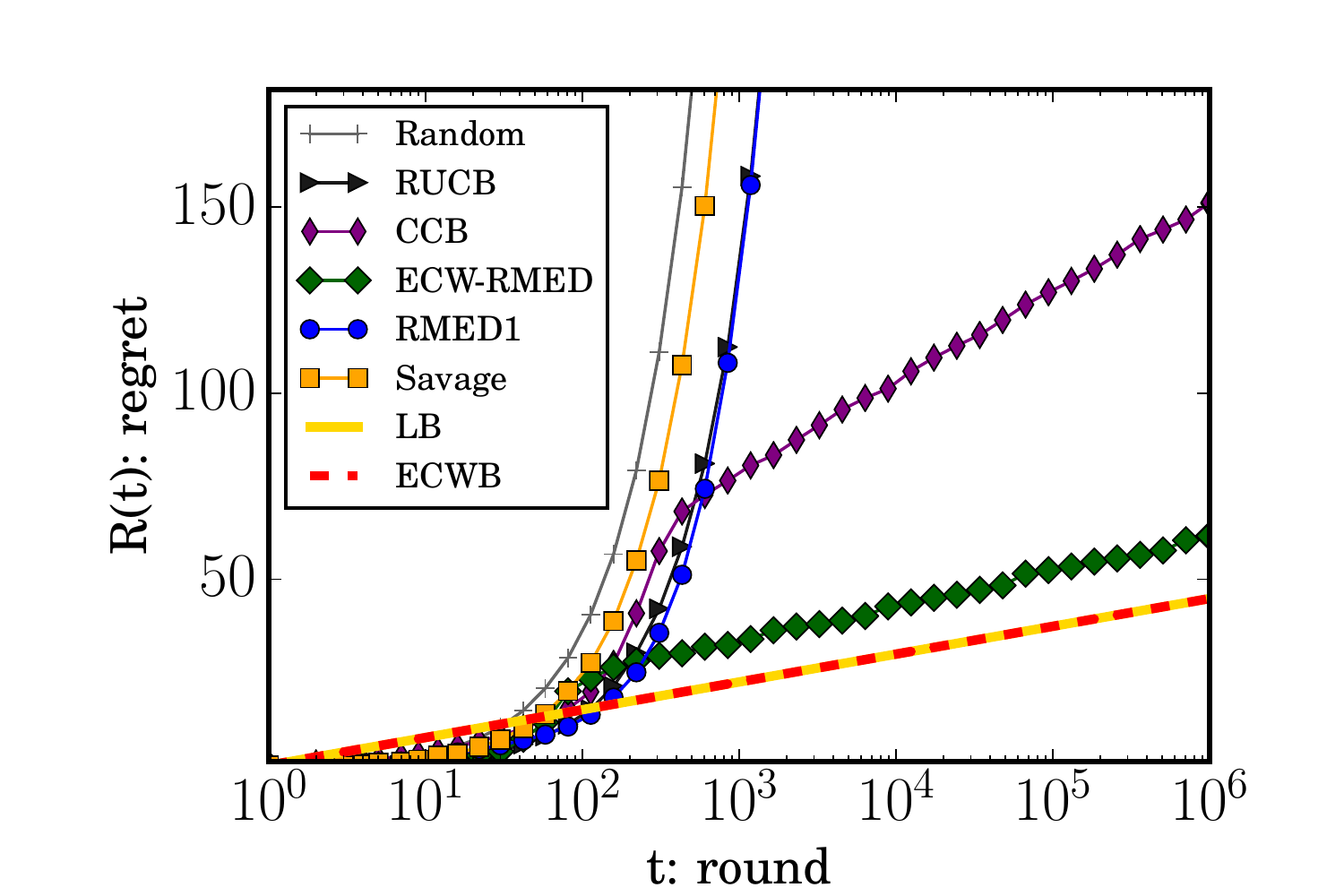}
 }
 \vspace{-0.501em}
  \end{minipage}\hfill

\end{center}
\vspace{-1em}
  \caption{Regret-round semilog plots of algorithms. The regrets are averaged over $100$ runs. LB and ECWB are the leading logarithmic terms of Theorems \ref{thm_regretlower} and \ref{thm_relax}, respectively. One can see that ECWB is very close to LB on the MSLR $K=16$ and the sushi datasets. We used the Gurobi LP solver for computing LB.}
 \label{fig_experiment}
\vspace{-1.0em}
\end{figure*}%

\subsection{Comparison of ECW-RMED and CCB}

In this section, we qualitatively discuss the improvement on the regret bound given by ECW-RMED. 

Let $\Delta:= \min_{(i,j)\in\SKP}{|\muij{i}{j}-1/2|}$. 
Theorem 3 in \citet{zoghicopeland} showed that CCB has an asymptotic regret bound\footnote{Here, we use a $\Delta$ that is a little bit looser than the one in the original bound of \citet{zoghicopeland} for the sake of discussion. In Table \ref{tbl_bounds}, we used the value of the original regret bound of CCB.} of 
\begin{equation}
\lo\left(\frac{K(C+L_1+1)}{\Delta^2}\log{T}\right).
\label{bound_ccb}
\end{equation}
On the other hand, $\optsetrelax{i_1}{\{\muij{i}{j}\}}$ includes
\begin{multline*}
\optvar{i}{j}/\KL(\muij{i}{j}, 1/2)\\
 = \begin{cases}
  1 & \text{if $i = i_1, j \in \SetH{i_1}$,} \\
  1/(\Li{i_2}-\Li{1}+1) & \text{if $i = i_2, j \in \SetO{i_2} \setminus \{i_1\}$,} \\
  0 & \text{otherwise},
 \end{cases}
\end{multline*}
which implies that ECW-RMED has a leading constant of 
\begin{multline}
\min_{i_1 \in \winners} \optconerelax{i_1}(\{\muij{i}{j}\}) \le \\
\frac{1}{\KL(1/2+\Delta,1/2)} \sum_{i_2 \ne i_1}\left(1 + \frac{\Li{i_2}}{\Li{i_2}-\Li{1}+1}\right).
\label{bound_relax_explicit}
\end{multline}
This bound is expressed in terms of the KL divergence instead of $\Delta^2 \le \KL(1/2+\Delta,1/2)/2$.
Furthermore, taking the maximum of \eqref{bound_relax_explicit} over $\{L_{i_2}\}_{i_2\neq i_1}$
with facts
\begin{align*}
\forall i_2\neq i_1,\,L_{1}\le L_{i_2}\le K,\qquad
%\nonumber\\
\sum_{i_2\neq i_1}\Li{i_2}\le \frac{K^2}{2}\nonumber,
%\label{const_relax}
\end{align*}
we see that
\begin{align}
\min_{i_1} \optconerelax{i_1}(\{\muij{i}{j}\})
%& \le \frac{K}{\KL(1/2+\Delta,1/2)} 	\left(1+\frac{K(\Li{1}+1)}{2(K-\Li{1}+1)}\right)\nonumber\\
&\le
\frac{K}{\KL(1/2+\Delta,1/2)}\left(\frac{\Li{1}+3}{2}+\frac{L_{1}^2}{K}\right).
\label{worst_ecw}
\end{align}
Therefore
the regret of ECW-RMED can be bounded independent of $C$
whereas \eqref{bound_ccb} contains a $O(CK)$ term.
Furthermore, the bound \eqref{worst_ecw} is tight only in the case that
$L_{i_2}$ is close to $L_{1}$ for $O(K)$ arms $i_2$,
which infrequently occurs in practice since $L_{i_2} \approx K/2$ on average.
In fact, if $L_{1}=o(K)$ and there exists $\rho\in(0,1/2)$ such that
$L_{i_2}\le \rho K$ for at most $o(K)$ arms $i_2$ then
we can bound \eqref{bound_relax_explicit} in the same way as \eqref{worst_ecw} by
\begin{equation*}
\min_{i_1} \optconerelax{i_1}(\{\muij{i}{j}\}) \le \frac{2K+o(K)}{\KL(1/2+\Delta,1/2)},
\end{equation*}
which is independent of $L_{1}$.

The only drawback of our analysis is the assumption on the uniqueness of the optimal solution, which is not very stringent. In our experiment, ECW-RMED performed well even when the optimal solution was not unique (MultiSol in Section \ref{sec_experiment}).

\vspace{-0.5em}
\subsection{On hyperparameters $\alpha$ and $\beta$}
\vspace{-0.5em}

CW/ECW-RMED have two hyperparameters $\alpha$ and $\beta$. 
The hyperparameter $\alpha$ is necessary in both theoretical and practical point of views. It urges the draw of each pair for $o(\log t)$ times to assure the quality of the estimator $\hatmut{i}{j}$.
On the other hand, we conjecture that the parameter $\beta$ is a theoretical artifact. 
Technically, the hyperparameter $\beta$ is required for bounding the regret when the quality of the estimation is low (i.e., inequality \eqref{ineq_xyzc_three} in  Appendix).
A very small or zero $\beta$ is practically sufficient: One can confirm that, setting $\beta=0$ yields almost the same results as shown in Section \ref{sec_experiment}.

\section{Numerical Experiment}
\label{sec_experiment}

To evaluate the empirical performance of the proposed algorithms, we conducted computer simulations with the following datasets (preference matrices).

\noindent\textbf{MSLR:} We tested submatrices of a $136 \times 136$ preference matrix from \citet{zoghiwsdm2015}, which is derived from the Microsoft Learning to Rank (MSLR) dataset \citep{mslr2010,DBLP:journals/ir/QinLXL10} that consists of relevance information between queries and documents with more than $30$K queries. \citet{zoghiwsdm2015} created a finite set of rankers, each of which corresponds to a ranking feature in the base dataset. The value $\muij{i}{j}$ is the probability that the ranker $i$ beats ranker $j$ based on the informational click model \citep{hofmann:tois13}. We randomly chose subsets of rankers in our experiments and made sub preference matrices. We excluded cases with extremely small gaps such that $|\muij{i}{j}-1/2|<0.005$ for $K=16$ or $|\muij{i}{j}-1/2|<0.0005$ for $K=64$. Furthermore, we selected the submatrices in which the Condorcet winner exists (Figure \ref{fig_mslrcond}) and the Condorcet winner does not exist (Figures \ref{fig_mslrnoncond} and \ref{fig_mslrlarge}).

\noindent\textbf{Sushi:} This dataset is based on the sushi preference dataset \cite{kamishimakdd2003} that contains the preferences of $5,000$ Japanese users as regards to $100$ types of sushi. We extracted $16$ types of sushi and converted them into a preference matrix with $\muij{i}{j}$ corresponding to the ratio of users who prefer sushi $i$ over $j$, which is shown in Table \ref{tbl_sushi} in Appendix.

\noindent\textbf{Gap} is the preference matrix of Table \ref{tbl_gap} in Appendix. This matrix is a corner case in which $(\argmin_{i_1} \optconerelax{i_1}(\{\muij{i}{j}\}))/(\argmin_{i_1} \optcone{\is}(\{\muij{i}{j}\})) > 100$.

\noindent\textbf{MultiSol} is the preference matrix of Table \ref{tbl_multisol} in Appendix. This matrix is an example in which the optimality condition in Theorem \ref{thm_relax} is violated.

Note that MLSR (Condorcet) and Sushi each have a Condorcet winner, whereas the others do not. The results with smaller preference matrices are shown in Appendix \ref{sec_addexperiment}.

\textbf{Algorithms:}
We compared the following algorithms: Random is a uniformly random sampling among pairs. Copeland SAVAGE with $\delta=1/T$ is the algorithm that is general enough to solve the Copeland dueling bandit problems and have $O(K^2 \log{T})$ regret bounds. We did not include PBR and RankEI because the two algorithms are reported to be consistently outperformed by other algorithms \cite{zoghicopeland}.
RUCB \citep{DBLP:conf/icml/ZoghiWMR14} with $\alpha=0.51$ and RMED1 \cite{DBLP:conf/colt/KomiyamaHKN15} are algorithms for solving Condorcet dueling bandit problems. These algorithms are not designed to find all instances of Copeland dueling bandit problems. The values of the hyperparameters of RMED1 are the same as in \citet{DBLP:conf/colt/KomiyamaHKN15}.
CCB \cite{zoghicopeland} with $\alpha=0.51$ and our ECW-RMED with $\algalpha=3.0$ and $\algbeta=0.01$ are algorithms designed for the Copeland dueling bandit problems. 

\textbf{Results:} 
Figure \ref{fig_experiment} plots the regrets of the algorithms. 
SAVAGE did not perform well for in any of the experiments. 
RMED1 performed best in MSLR (Condorcet). However, in datasets such as MSLR (non-Condorcet) and MultiSol where the Condorcet winner does not exist, it suffered a large regret. RUCB did not perform better than RMED1 and showed a similar tendency. These observations support the hypothesis that these algorithms are not capable of finding a Copeland winner.
CCB performed similarly to RUCB in many datasets and outperformed RUCB for the datasets without a Condorcet winner.
ECW-RMED significantly outperformed CCB and in all datasets, including Gap in which the uniqueness assumption of Theorem \ref{thm_relax} is violated. In particular, in MSLR non-Condorcet dataset with $K=16$, the regret of ECW-RMED was more than three times smaller than that of CCB. The slope of ECW-RMED in many of the datasets is close to ECWB when $T$ is large, which is consistent with our analysis.

\section{Conclusion}

We studied the stochastic dueling bandit problem. The hardness of the problem of recommending Copeland winners was uncovered by deriving a lower bound of the regret. CW-RMED, an asymptotically optimal algorithm, was proposed. Moreover, ECW-RMED, a close-to-optimal algorithm, was proposed and an efficient computation method of it is given. ECW-RMED significantly outperformed the state-of-the-art algorithms in an experiment. 

\section*{Acknowledgements}

This work was supported in part by JSPS KAKENHI Grant Number 15J09850 and 16H00881.

\clearpage
\bibliography{./bibs/manual.bib}
%\bibliography{example_paper}
\bibliographystyle{icml2016}

\clearpage

\onecolumn 
\appendix

\section{Preference Matrices in the Experiment}

The following table shows the preference matrices that are used in the numerical experiment in Section \ref{sec_experiment}.

\begin{table*}[h]
\begin{footnotesize}
\caption{Preference matrices in the experiment. In each matrix, the $ij$ element is $\muij{i}{j}$.}
\begin{center}
\subtable[Sushi][Sushi. Rows are 1. mildly fatty tuna, 2. fatty tuna, 3. salmon, 4. tuna, 5. salmon roe, 6. sea bream, 7. sea eel, 8. scallop, 9. squid, 10. horse mackerel, 11. eel, 12. abalone, 13. mackerel, 14. squid feet, 15. Tori clam, 16. squilla, respectively. \label{tbl_sushi}]{
  \begin{tabular}{|l|l|l|l|l|l|l|l|l|l|l|l|l|l|l|l|}  \hline
0.5 &\hspace{-0.1em}0.512 &\hspace{-0.1em}0.622 &\hspace{-0.1em}0.655 &\hspace{-0.1em}0.698 &\hspace{-0.1em}0.726 &\hspace{-0.1em}0.711 &\hspace{-0.1em}0.708 &\hspace{-0.1em}0.749 &\hspace{-0.1em}0.8 &\hspace{-0.1em}0.741 &\hspace{-0.1em}0.783 &\hspace{-0.1em}0.847 &\hspace{-0.1em}0.817 &\hspace{-0.1em}0.854 &\hspace{-0.1em}0.868 \\ \hline
0.488 &\hspace{-0.1em}0.5 &\hspace{-0.1em}0.602 &\hspace{-0.1em}0.683 &\hspace{-0.1em}0.652 &\hspace{-0.1em}0.776 &\hspace{-0.1em}0.663 &\hspace{-0.1em}0.683 &\hspace{-0.1em}0.738 &\hspace{-0.1em}0.709 &\hspace{-0.1em}0.786 &\hspace{-0.1em}0.802 &\hspace{-0.1em}0.83 &\hspace{-0.1em}0.85 &\hspace{-0.1em}0.871 &\hspace{-0.1em}0.873 \\ \hline
0.378 &\hspace{-0.1em}0.398 &\hspace{-0.1em}0.5 &\hspace{-0.1em}0.528 &\hspace{-0.1em}0.554 &\hspace{-0.1em}0.533 &\hspace{-0.1em}0.534 &\hspace{-0.1em}0.591 &\hspace{-0.1em}0.573 &\hspace{-0.1em}0.593 &\hspace{-0.1em}0.661 &\hspace{-0.1em}0.705 &\hspace{-0.1em}0.734 &\hspace{-0.1em}0.672 &\hspace{-0.1em}0.787 &\hspace{-0.1em}0.822 \\ \hline
0.345 &\hspace{-0.1em}0.317 &\hspace{-0.1em}0.472 &\hspace{-0.1em}0.5 &\hspace{-0.1em}0.553 &\hspace{-0.1em}0.619 &\hspace{-0.1em}0.566 &\hspace{-0.1em}0.641 &\hspace{-0.1em}0.675 &\hspace{-0.1em}0.687 &\hspace{-0.1em}0.665 &\hspace{-0.1em}0.696 &\hspace{-0.1em}0.803 &\hspace{-0.1em}0.823 &\hspace{-0.1em}0.796 &\hspace{-0.1em}0.844 \\ \hline
0.302 &\hspace{-0.1em}0.348 &\hspace{-0.1em}0.446 &\hspace{-0.1em}0.447 &\hspace{-0.1em}0.5 &\hspace{-0.1em}0.513 &\hspace{-0.1em}0.524 &\hspace{-0.1em}0.518 &\hspace{-0.1em}0.608 &\hspace{-0.1em}0.538 &\hspace{-0.1em}0.643 &\hspace{-0.1em}0.61 &\hspace{-0.1em}0.695 &\hspace{-0.1em}0.672 &\hspace{-0.1em}0.681 &\hspace{-0.1em}0.775 \\ \hline
0.274 &\hspace{-0.1em}0.224 &\hspace{-0.1em}0.467 &\hspace{-0.1em}0.381 &\hspace{-0.1em}0.487 &\hspace{-0.1em}0.5 &\hspace{-0.1em}0.513 &\hspace{-0.1em}0.559 &\hspace{-0.1em}0.575 &\hspace{-0.1em}0.621 &\hspace{-0.1em}0.591 &\hspace{-0.1em}0.701 &\hspace{-0.1em}0.702 &\hspace{-0.1em}0.787 &\hspace{-0.1em}0.829 &\hspace{-0.1em}0.811 \\ \hline
0.289 &\hspace{-0.1em}0.337 &\hspace{-0.1em}0.466 &\hspace{-0.1em}0.434 &\hspace{-0.1em}0.476 &\hspace{-0.1em}0.487 &\hspace{-0.1em}0.5 &\hspace{-0.1em}0.559 &\hspace{-0.1em}0.553 &\hspace{-0.1em}0.613 &\hspace{-0.1em}0.564 &\hspace{-0.1em}0.607 &\hspace{-0.1em}0.703 &\hspace{-0.1em}0.735 &\hspace{-0.1em}0.736 &\hspace{-0.1em}0.801 \\ \hline
0.292 &\hspace{-0.1em}0.317 &\hspace{-0.1em}0.409 &\hspace{-0.1em}0.359 &\hspace{-0.1em}0.482 &\hspace{-0.1em}0.441 &\hspace{-0.1em}0.441 &\hspace{-0.1em}0.5 &\hspace{-0.1em}0.556 &\hspace{-0.1em}0.527 &\hspace{-0.1em}0.562 &\hspace{-0.1em}0.58 &\hspace{-0.1em}0.668 &\hspace{-0.1em}0.805 &\hspace{-0.1em}0.777 &\hspace{-0.1em}0.767 \\ \hline
0.251 &\hspace{-0.1em}0.262 &\hspace{-0.1em}0.427 &\hspace{-0.1em}0.325 &\hspace{-0.1em}0.392 &\hspace{-0.1em}0.425 &\hspace{-0.1em}0.447 &\hspace{-0.1em}0.444 &\hspace{-0.1em}0.5 &\hspace{-0.1em}0.512 &\hspace{-0.1em}0.548 &\hspace{-0.1em}0.542 &\hspace{-0.1em}0.612 &\hspace{-0.1em}0.786 &\hspace{-0.1em}0.71 &\hspace{-0.1em}0.685 \\ \hline
0.2 &\hspace{-0.1em}0.291 &\hspace{-0.1em}0.407 &\hspace{-0.1em}0.313 &\hspace{-0.1em}0.462 &\hspace{-0.1em}0.379 &\hspace{-0.1em}0.387 &\hspace{-0.1em}0.473 &\hspace{-0.1em}0.488 &\hspace{-0.1em}0.5 &\hspace{-0.1em}0.543 &\hspace{-0.1em}0.579 &\hspace{-0.1em}0.613 &\hspace{-0.1em}0.718 &\hspace{-0.1em}0.685 &\hspace{-0.1em}0.747 \\ \hline
0.259 &\hspace{-0.1em}0.214 &\hspace{-0.1em}0.339 &\hspace{-0.1em}0.335 &\hspace{-0.1em}0.357 &\hspace{-0.1em}0.409 &\hspace{-0.1em}0.436 &\hspace{-0.1em}0.438 &\hspace{-0.1em}0.452 &\hspace{-0.1em}0.457 &\hspace{-0.1em}0.5 &\hspace{-0.1em}0.564 &\hspace{-0.1em}0.625 &\hspace{-0.1em}0.618 &\hspace{-0.1em}0.702 &\hspace{-0.1em}0.684 \\ \hline
0.217 &\hspace{-0.1em}0.198 &\hspace{-0.1em}0.295 &\hspace{-0.1em}0.304 &\hspace{-0.1em}0.39 &\hspace{-0.1em}0.299 &\hspace{-0.1em}0.393 &\hspace{-0.1em}0.42 &\hspace{-0.1em}0.458 &\hspace{-0.1em}0.421 &\hspace{-0.1em}0.436 &\hspace{-0.1em}0.5 &\hspace{-0.1em}0.542 &\hspace{-0.1em}0.644 &\hspace{-0.1em}0.7 &\hspace{-0.1em}0.733 \\ \hline
0.153 &\hspace{-0.1em}0.17 &\hspace{-0.1em}0.266 &\hspace{-0.1em}0.197 &\hspace{-0.1em}0.305 &\hspace{-0.1em}0.298 &\hspace{-0.1em}0.297 &\hspace{-0.1em}0.332 &\hspace{-0.1em}0.388 &\hspace{-0.1em}0.387 &\hspace{-0.1em}0.375 &\hspace{-0.1em}0.458 &\hspace{-0.1em}0.5 &\hspace{-0.1em}0.577 &\hspace{-0.1em}0.607 &\hspace{-0.1em}0.596 \\ \hline
0.183 &\hspace{-0.1em}0.15 &\hspace{-0.1em}0.328 &\hspace{-0.1em}0.177 &\hspace{-0.1em}0.328 &\hspace{-0.1em}0.213 &\hspace{-0.1em}0.265 &\hspace{-0.1em}0.195 &\hspace{-0.1em}0.214 &\hspace{-0.1em}0.282 &\hspace{-0.1em}0.382 &\hspace{-0.1em}0.356 &\hspace{-0.1em}0.423 &\hspace{-0.1em}0.5 &\hspace{-0.1em}0.578 &\hspace{-0.1em}0.637 \\ \hline
0.146 &\hspace{-0.1em}0.129 &\hspace{-0.1em}0.213 &\hspace{-0.1em}0.204 &\hspace{-0.1em}0.319 &\hspace{-0.1em}0.171 &\hspace{-0.1em}0.264 &\hspace{-0.1em}0.223 &\hspace{-0.1em}0.29 &\hspace{-0.1em}0.315 &\hspace{-0.1em}0.298 &\hspace{-0.1em}0.3 &\hspace{-0.1em}0.393 &\hspace{-0.1em}0.422 &\hspace{-0.1em}0.5 &\hspace{-0.1em}0.586 \\ \hline
0.132 &\hspace{-0.1em}0.127 &\hspace{-0.1em}0.178 &\hspace{-0.1em}0.156 &\hspace{-0.1em}0.225 &\hspace{-0.1em}0.189 &\hspace{-0.1em}0.199 &\hspace{-0.1em}0.233 &\hspace{-0.1em}0.315 &\hspace{-0.1em}0.253 &\hspace{-0.1em}0.316 &\hspace{-0.1em}0.267 &\hspace{-0.1em}0.404 &\hspace{-0.1em}0.363 &\hspace{-0.1em}0.414 &\hspace{-0.1em}0.5 \\ \hline
 \end{tabular}
}
\subtable[Gap][Gap]{
  \begin{tabular}{|l|l|l|l|l|}  \hline
0.5 & 0.8 & 0.8 & 0.51 & 0.2 \\ \hline  
0.2 & 0.5 & 0.8 & 0.2 & 0.8 \\ \hline  
0.2 & 0.2 & 0.5 & 0.8 & 0.8 \\ \hline  
0.49 & 0.8 & 0.2 & 0.5 & 0.2 \\ \hline  
0.8 & 0.2 & 0.2 & 0.8 & 0.5 \\ \hline  
 \end{tabular}
 \label{tbl_gap}
}
\subtable[MultiSol][MultiSol]{
 \begin{tabular}{|l|l|l|l|l|}  \hline
0.5 & 0.2 & 0.8 & 0.8 & 0.8 \\ \hline  
0.8 & 0.5 & 0.2 & 0.8 & 0.8 \\ \hline  
0.2 & 0.8 & 0.5 & 0.8 & 0.8 \\ \hline  
0.2 & 0.2 & 0.2 & 0.5 & 0.6 \\ \hline  
0.2 & 0.2 & 0.2 & 0.4 & 0.5 \\ \hline  
  \end{tabular}
 \label{tbl_multisol}
}
\subtable[ArXiv][ArXiv]{
  \begin{tabular}{|l|l|l|l|l|l|}  \hline
0.50 & 0.55 & 0.55 & 0.54 & 0.61 & 0.61 \\ \hline  
0.45 & 0.50 & 0.55 & 0.55 & 0.58 & 0.60 \\ \hline  
0.45 & 0.45 & 0.50 & 0.54 & 0.51 & 0.56 \\ \hline  
0.46 & 0.45 & 0.46 & 0.50 & 0.54 & 0.50 \\ \hline  
0.39 & 0.42 & 0.49 & 0.46 & 0.50 & 0.51 \\ \hline  
0.39 & 0.40 & 0.44 & 0.50 & 0.49 & 0.50 \\ \hline  
  \end{tabular}
 \label{tbl_arxiv}
}
\subtable[MSLR (fixed, $K=5$, Condorcet)][MSLR (fixed, $K=5$, Condorcet)]{
 \begin{tabular}{|l|l|l|l|l|}  \hline
0.5   & 0.535 & 0.613 & 0.757 & 0.765 \\ \hline 
0.465 & 0.5   & 0.580 & 0.727 & 0.738 \\ \hline 
0.387 & 0.420 & 0.5   & 0.659 & 0.669 \\ \hline 
0.243 & 0.276 & 0.341 & 0.5   & 0.510 \\ \hline 
0.235 & 0.262 & 0.331 & 0.490 & 0.5   \\ \hline 
  \end{tabular}
 \label{tbl_mslrcondfive}
}
\subtable[MSLR Fixed ($K=5$, non-Condorcet)][MSLR Fixed ($K=5$, non-Condorcet)]{
 \begin{tabular}{|l|l|l|l|l|}  \hline
0.5   & 0.484 & 0.519 & 0.529 & 0.518 \\ \hline 
0.516 & 0.5   & 0.481 & 0.530 & 0.539 \\ \hline 
0.481 & 0.519 & 0.5   & 0.504 & 0.512 \\ \hline 
0.471 & 0.470 & 0.496 & 0.5   & 0.503 \\ \hline 
0.482 & 0.461 & 0.488 & 0.497 & 0.5   \\ \hline 
 \end{tabular}
 \label{tbl_mslrnoncondfive}
}
\end{center}
%\vspace{-1em}

\end{footnotesize}
\end{table*}

\section{Additional Experiment}
\label{sec_addexperiment}

We conducted additional simulations with the following datasets.

\noindent\textbf{ArXiv} is a preference matrix based on the six retrieval functions in the full-text search engine of ArXiv.org \cite{DBLP:conf/icml/YueJ11} shown in Table \ref{tbl_arxiv}, where an order among arms exists. Although the fact $\muij{4}{6} = 1/2$ violates our assumption, the Copeland winner is arguably arm $1$.

\noindent\textbf{Cyclic} is the preference matrix of Table \ref{tbl_cyclic}.

\noindent\textbf{MSLR Fixed} are the two matrices of size $5\times5$ provided by \citet{zoghicopeland} shown in Table \ref{tbl_mslrcondfive} and \ref{tbl_mslrnoncondfive}. One matrix has a Condorcet winner, whereas the other does not. We include these matrices to compare our results with their ones.

The results of the simulations are shown in Figure \ref{fig_experiment_add}. 

\begin{figure*}[t!]
\begin{center}
  \setlength{\subfigwidth}{.50\linewidth}
  \addtolength{\subfigwidth}{-.50\subfigcolsep}
  \begin{minipage}[t]{\subfigwidth}
  \centering
 \subfigure[ArXiv]{
 \includegraphics[scale=0.51]{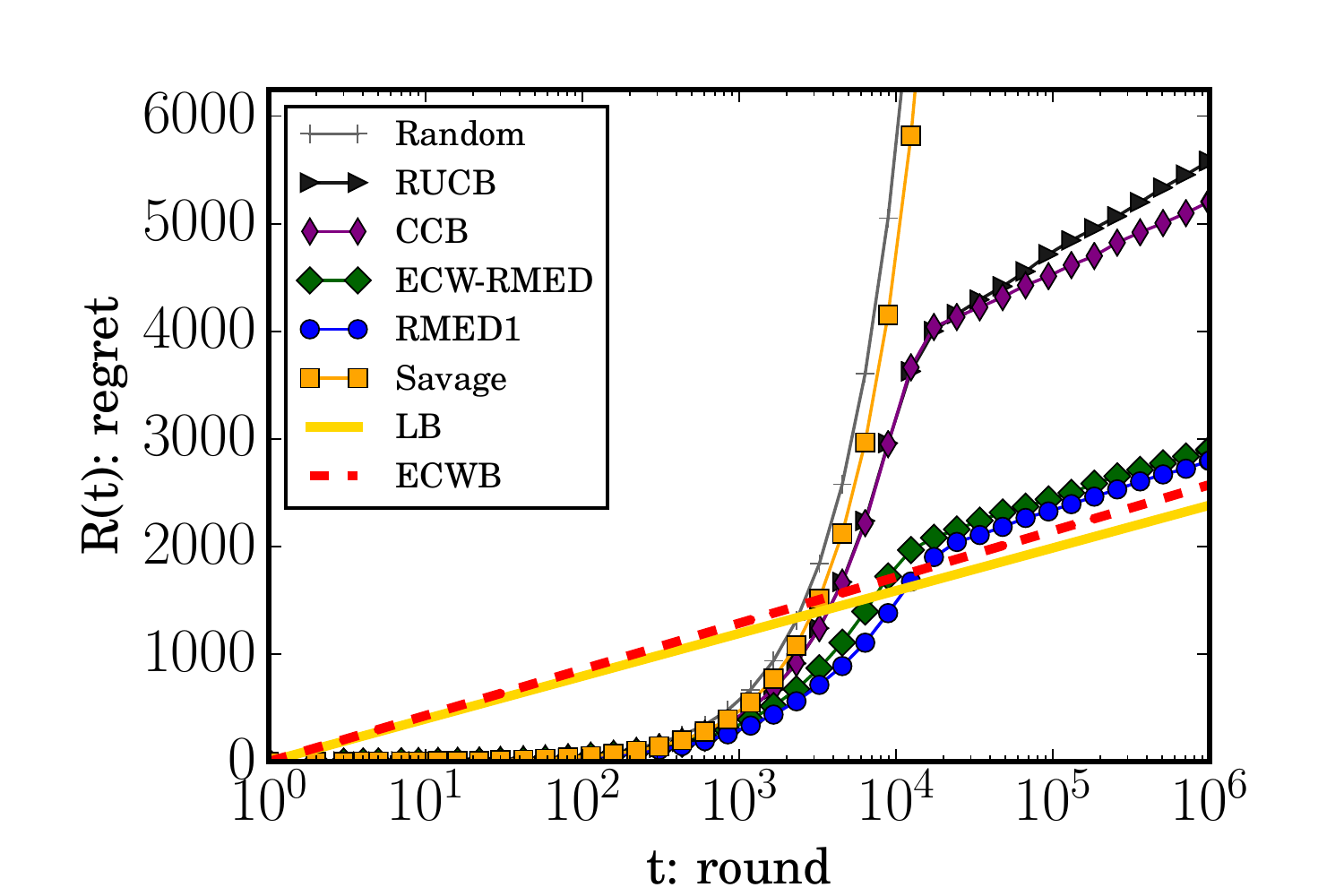}
 }
 \vspace{-0.501em}
 \end{minipage}\hfill
  \begin{minipage}[t]{\subfigwidth}
  \centering
 \subfigure[Cyclic]{\includegraphics[scale=0.51]{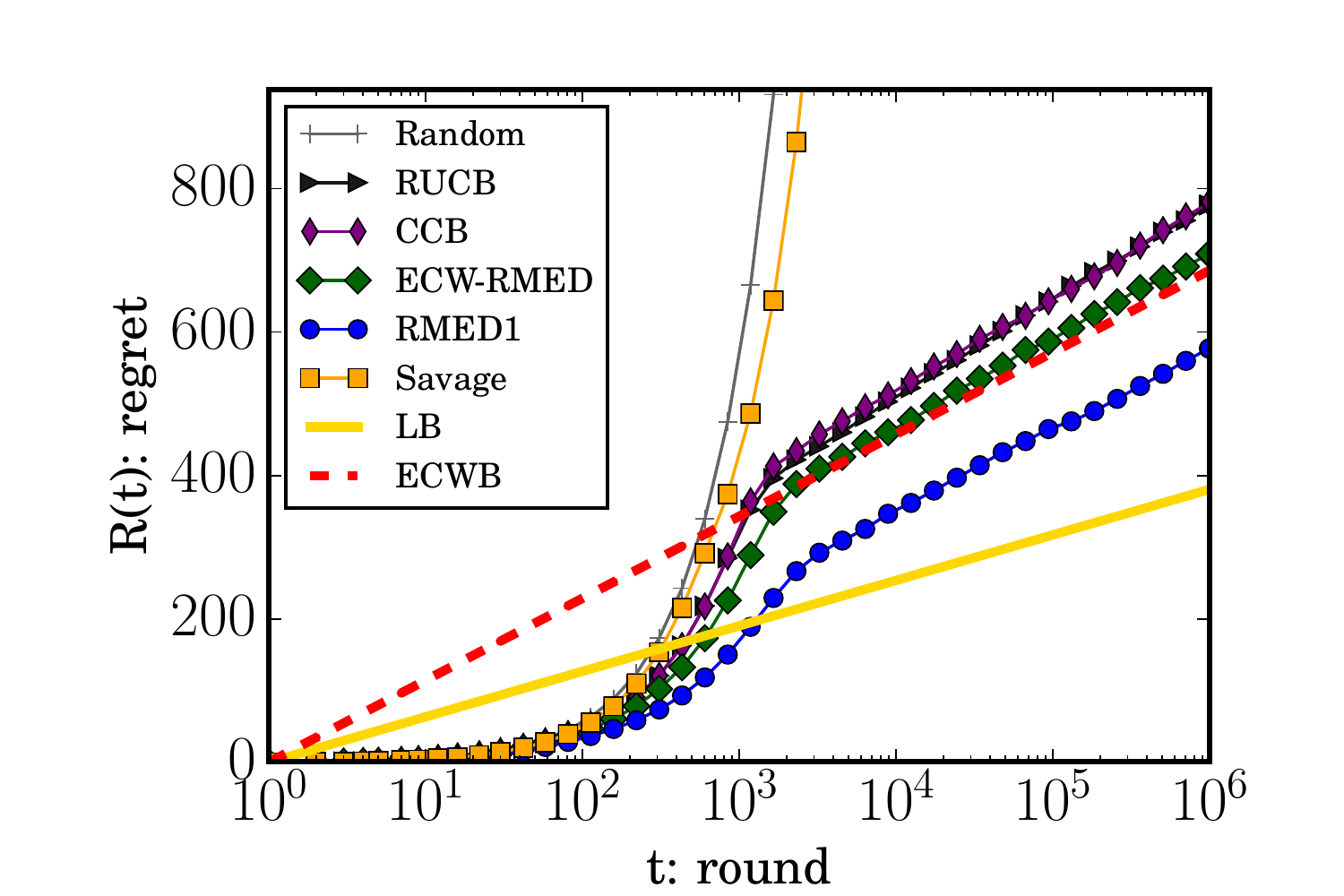}
 }
 \vspace{-0.501em}
  \end{minipage}\hfill

  \begin{minipage}[t]{\subfigwidth}
  \centering
 \subfigure[MSLR Fixed ($K=5$, Condorcet)]{
 \includegraphics[scale=0.51]{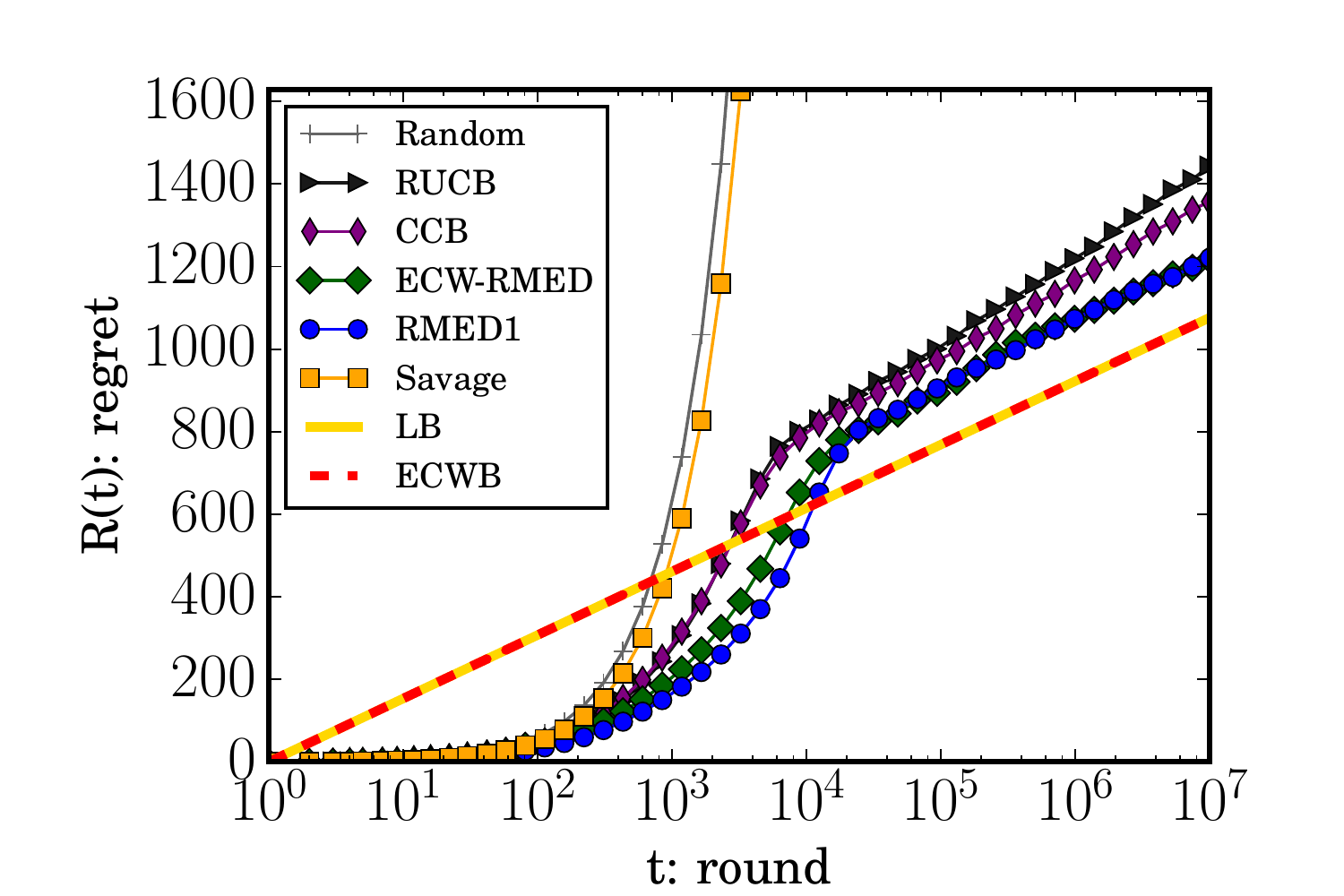}
 \label{fig_mslrcondfive}
 }
 \vspace{-0.501em}
 \end{minipage}\hfill
  \begin{minipage}[t]{\subfigwidth}
  \centering
 \subfigure[MSLR Fixed ($K=5$, non-Condorcet)]{\includegraphics[scale=0.51]{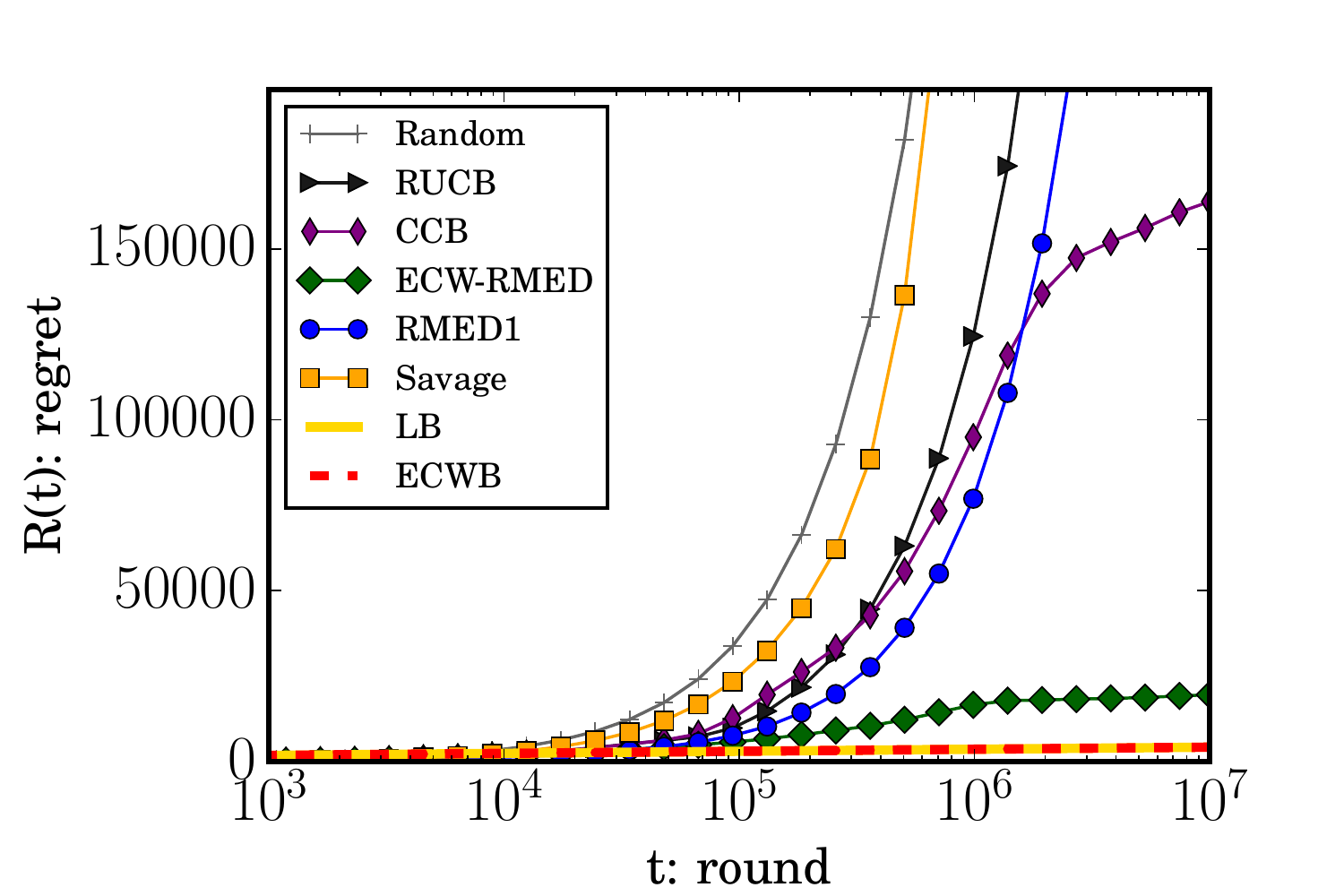}
 \label{fig_mslrnoncondfive}
 }
 \vspace{-0.501em}
  \end{minipage}\hfill
\end{center}
%\vspace{-1em}
  \caption{Regret-round semilog plots of algorithms. The regrets are averaged over $100$ runs. The algorithms, LB, and ECWB in the plots are the same as in the main text.}
%\vspace{-1.0em}
 \label{fig_experiment_add}
\end{figure*}%

\begin{table}[b!]
\caption{Comparison of leading logarithmic constants of regret bounds in the cyclic dataset.}
\begin{center}
{\renewcommand\arraystretch{2.0}
  \begin{tabular}{|c|c|c|} \hline
   \shortstack{\\Optimal: \\ CW-RMED} & ECW-RMED & \shortstack{CCB} \\\hline\hline
   $27.5$ & $49.7$ & $1600$ \\\hline
  \end{tabular}
}
\end{center}
\label{tbl_bounds_cyclic}
\end{table}

\section{Comparison of Regret Bounds}
\label{sec_quantcompregret}

In this section, we clarify differences among the regret bounds of CW-RMED, ECW-RMED and CCB by calculating them in the cyclic preference matrix (Table \ref{tbl_cyclic}). In the cyclic preference matrix, we have $C=1, L_1 = 0$, and $\Delta = \min_{(i,j)\in\SKP}{|\muij{i}{j}-1/2|} = 0.1$. Table \ref{tbl_bounds_cyclic} shows the regret bounds of the three algorithms. These bounds are calculated as follows.
First, the regret bound of CW-RMED (inequality \eqref{optcond}) states that the risk of  arm $1$ being a non-Copeland winner is smaller than $\log T$: It requires $\NT{1}{2}, \NT{1}{3}, \NT{1}{4} \ge (\log{T})/(2\KL(0.6, 0.5))$ and $\NT{2}{3}, \NT{3}{4}, \NT{4}{2} \ge (\log{T})/(2\KL(0.9, 0.5))$.
Second, the regret bound of ECW-RMED (inequality \eqref{lloses}) requires that (i) the arm $1$ beats all other arms in $\SetH{1}$ and (ii) the other arms loses at least $L_1$ times: It requires $\NT{1}{2}, \NT{1}{3}, \NT{1}{4} \ge (\log{T})/\KL(0.6, 0.5)$.
Finally, the regret bound of CCB ($(\Regret(T)/\log{T}) = \frac{2K(C+L_1+1)}{\Delta^2} = 1600$) is much larger 
%than ones of CW/ECW-RMED
 because it corresponds to the exploration for checking that (i) arm $1$ wins against all other arms ($C K$ pairs) and (ii) the other arms loses at least $L_1+1$ times ($(L_1+1)K$ pairs). Moreover, (iii) it may compare all pairs that are required to confirm (i)--(ii) for $2 \log{T}/\Delta^2$ times.

\section{Facts}
\label{sec_facts}

The following facts are frequently used in this paper. 
Fact \ref{fact:chernoff} is a concentration inequality that bounds the tail probability on the empirical means. Fact \ref{fact_pinsker} is used to bound the KL divergence from below. Fact \ref{fact:minimumdivergencediff} is later used in the proof of Lemma \ref{lem_copeest}.

\begin{fact} {\rm (The Chernoff bound)}\\
Let $X_1,\dots,X_n$ be i.i.d.\,binary random variables.
Let $\hat{X} = \frac{1}{n}\sum_{i=1}^n X_i$ and $\mu = \Expect[\hat{X}]$.
Then, for any $\epsilon > 0$,
\begin{equation*}
  \Prob( \hat{X} \geq \mu + \epsilon ) \leq \exp{\left( - \KL(\mu+\epsilon, \mu) n \right)}
\end{equation*}
and
\begin{equation*}
  \Prob( \hat{X} \leq \mu - \epsilon ) \leq \exp{\left( - \KL(\mu-\epsilon, \mu) n \right)}.
\end{equation*}
\label{fact:chernoff}
\end{fact}

\begin{fact} {\rm (The Pinsker's inequality)}\\
For $p,q \in (0,1)$, the KL divergence between two Bernoulli distributions is bounded as: 
\begin{equation*}
\KL(p,q) \geq 2 (p-q)^2. \label{ineq:pinsker}
\end{equation*}
\label{fact_pinsker}
\end{fact}

\begin{fact} {\rm (Lemma 13 of \citealt{HondaDMED})}\\
For any $\mu$ and $\mu_2$ satisfying $0 < \mu_2 < \mu < 1$. Let $C_1(\mu, \mu_2) = (\mu - \mu_2)^2 / (2 \mu (1- \mu_2))$. Then, for any $\mu_3 \leq \mu_2$,
\begin{equation*}
  \KL(\mu_3, \mu) - \KL(\mu_3, \mu_2) \geq C_1(\mu, \mu_2) > 0.
\end{equation*}
\label{fact:minimumdivergencediff}
\end{fact}

\section{Proofs on the Regret Lower Bound}
\label{sec_prooflower}

In this section, we prove Lemma \ref{lem_lbnumarm} and Theorem \ref{thm_regretlower}.
In proofs, we frequently denote $\EA, \EB$ instead of $\EA \cap \EB$ for two events $\EA$ and $\EB$.

\begin{proof}[Proof of Lemma \ref{lem_lbnumarm}]
Let $\lbdelta>0$ be arbitrary. For $i_1 \in \winners$, $i_2 \neq i_1$, $l \in \{\max\{0,\Lone-1\},\dots,\Ltwo\}$, $I \in \SubPowSetH{i_1}{l+1-\Li{i_1}}$, $S \in \SubPowSetORem{i_2}{\max\{0,\Li{i_2}-l-\Ind\{i_1 \in I\}\}}{i_1}$, let 
\begin{equation*}
 \sumdiv{i_1, i_2, l, I, S} := \sum_{(i,j) \in \SetIS} \KL(\muij{i}{j}, 1/2) \NT{i}{j}
\end{equation*}
and
\begin{align*}
 \EE_{i_1, i_2, l, I, S}(T) & := \{ \sumdiv{i_1, i_2, l, I, S} \le (1 - \lbdelta) \log{T} \} \\
 \EA(T) & := \cap_{i_1} \cup_{i_2}\,\cup_{l}\,\cup_{I}\,\cup_{S}\,\EE_{i_1, i_2, l, I, S}(T).
\end{align*}
In the following we prove
\begin{equation*}
 \lim_{T \rightarrow \infty} \Prob[ \EA(T) ] = 0,
%\label{ineq_lowergoal}
\end{equation*}
which implies Lemma \ref{lem_lbnumarm}.
Let
\begin{equation*}
 \nSum{i}(T) := \sum_{j \in \arms} \NT{i}{j} 
\end{equation*}
and
\begin{equation*}
 \EB_i(T) := \left\{ \nSum{i}(T) = \max_{i' \in \arms} \nSum{i'}(T) \right\}.
\end{equation*}
Note that $\cup_{i \in \arms} \EB_i(T)$ always occurs.
Since $\EB_{i_1}(T)$ implies $\nSum{i_1}(T) \ge T/K = \Omega(T)$, consistency requires $\Prob[\EB_i(T)] = o(1)$ for each $i \notin \winners$ and thus 
\begin{equation}
 \Prob[\cup_{i \in \winners} \EB_i(T)] = 1 - o(1).
 \label{ineq_ebcopelands}
\end{equation}

Let $\lbepsilon>0$ be sufficiently small.
Consider a modified preference matrix $\Matp_{i_1,i_2,l,I,S} := \{\muijd{i}{j}\} = \{\muijd{i}{j}(i_1,i_2,l,I,S)\}$ such that, for each pair $(i,j)$ in $\SetIS$, if $\muij{i}{j} > 1/2$ then $\muijd{i}{j} < 1/2$ otherwise (i.e., if $\muij{i}{j} < 1/2$)  $\muijd{i}{j} > 1/2$ such that
\begin{equation}
 \KL(\muij{i}{j}, \muijd{i}{j}) = \KL(\muij{i}{j}, 1/2) (1 + \lbepsilon). \label{ineq_iwinmod}
\end{equation}
Such a $\muijd{i}{j}$ for each pair $(i,j)$ uniquely exists for sufficiently small $\lbepsilon$.
For each pair $(i,j)$ that are not involved in $\SetIS$, we set $\muijd{i}{j} = \muij{i}{j}$.
Let $\Expectp = \Expectp_{i_1,i_2,l,I,S}$, $\Probp = \Probp_{i_1,i_2,l,I,S}$ be the expectation and probability of the algorithm with respect to the modified preference matrix $\Matp_{i_1,i_2,l,I,S}$.
Let $\Lid{i} = \{j \in \arms: j \neq i, \muijd{i}{j} < 1/2\}$ be the number of arms that beat $i$ in the modified game. In the modified game, arm $i_1$ is not a Copeland winner because $\Lid{i_1} = l+1$ and $\Lid{i_2} = l$. 
Let $\Xij^m \in \{0,1\}$ be the result of $m$-th draw of the pair ($i$, $j$), 
\begin{equation*}
  \hKLij{i}{j}(\nij{i}{j}) = \sum_{m=1}^{\nij{i}{j}} \log{\left(\frac{\Xij^m \muij{i}{j} + (1 - \Xij^m)(1-\muij{i}{j})}{\Xij^m \muijd{i}{j} + (1 - \Xij^m)(1-\muijd{i}{j})}\right)},
\end{equation*}
and $\hKL(\{\nij{i}{j}\}_{(i,j) \in \SetIS}) = \sum_{(i,j) \in \SetIS} \hKLij{i}{j}(\nij{i}{j})$.
Let $\lbepsilontwo > 0$ and 
\begin{align*}
 \ED_{i_1, i_2, l, I, S}(T) &:= \left\{ \hKL(\{\nij{i}{j}\}_{(i,j) \in \SetIS}) < (1 - \lbepsilontwo) \log{T} \right\}.
\end{align*}

For any $i_1, i_2, l, I, S$, 
\begin{align*}
\lefteqn{
 \Prob[\EE_{i_1, i_2, l, I, S}(T) \cap \ED_{i_1, i_2, l, I, S}^c(T)]
}\\
&\le \Prob[ \sumdiv{i_1, i_2, l, I, S} \le (1 - \lbdelta) \log{T}, \hKL(\{\nij{i}{j}\}) > (1 - \lbepsilontwo) \log{T}] \\
&= \Prob[ \sum_{(i,j) \in \SetIS} \KL(\muij{i}{j}, \muijd{i}{j}) \NT{i}{j} \le (1 - \lbdelta)(1+\lbepsilon) \log{T}, \hKL(\{\nij{i}{j}\}) > (1 - \lbepsilontwo) \log{T}] \\
&\le \Prob\left[ \max_{\{\nij{i}{j}\}_{(i,j) \in \SetIS} \in \Natural^{|\SetIS|}, \sum_{(i,j) \in \SetIS} \KL(\muij{i}{j}, \muijd{i}{j}) \nij{i}{j} \le (1 - \lbdelta)(1 + \lbepsilon) \log{T}} \hKL(\{\nij{i}{j}\}) > (1 - \lbepsilontwo) \log{T} \right].
\end{align*}
Note that,
\begin{equation*}
\max_{1 \le n \le N} \hKLij{i}{j}(n)
\end{equation*}
is the maximum sum of positive-mean random variables, and thus converges to its average. Namely,
\begin{equation*}
\lim_{N \rightarrow \infty} \max_{1 \le n \le N} \hKLij{i}{j}(n)/N = \KL(\muij{i}{j}, \muijd{i}{j}) \mathrm{\hspace{2em}a.s.}
\end{equation*}
and thus
\begin{multline*}
 \limsup_{T \rightarrow \infty} 
\frac{ 
 \max_{\{\nij{i}{j}\}_{(i,j) \in \SetIS} \in \Natural^{|\SetIS|}, \sum_{(i,j) \in \SetIS} \KL(\muij{i}{j}, \muijd{i}{j}) \nij{i}{j} \le (1 - \lbdelta)(1 + \lbepsilon) \log{T}} \hKL(\{\nij{i}{j}\})
}{\log{T}} \\ \le (1 - \lbdelta)(1 + \lbepsilon) \mathrm{\hspace{2em}a.s.}
\end{multline*}
Take $\lbepsilon = \lbepsilon(\lbdelta)$ and $\lbepsilontwo = \lbepsilontwo(\lbdelta)$ such that $(1 - \lbdelta)(1 + \lbepsilon) < (1 - \lbepsilontwo)$,
and as a result 
\begin{equation*}
\lim_{T \rightarrow \infty} \Prob\left[ \max_{\{\nij{i}{j}\}_{(i,j) \in \SetIS} \in \Natural^{|\SetIS|}, \sum_{(i,j) \in \SetIS} \KL(\muij{i}{j}, \muijd{i}{j}) \nij{i}{j} \le (1 - \lbdelta)(1 + \lbepsilon) \log{T}} \hKL(\{\nij{i}{j}\}) > (1 - \lbepsilontwo) \log{T} \right] = 0,
\end{equation*}
which leads to 
\begin{equation}
 \Prob[\EE_{i_1, i_2, l, I, S}(T) \cap \ED_{i_1, i_2, l, I, S}^c(T)] = o(1)
\label{ineq_lem_stepone}
\end{equation}
 as a function of $T$.

Note that the consistency requires
\begin{equation*}
  \Probp_{i_1, i_2, l, I, S}\{\EB_{i_1}(T)\} = o(T^{a-1})
\end{equation*}
for any $a>0$. Take $a<\lbepsilontwo$.
For any $i_1, i_2, l, I, S$, 
\begin{align}
\lefteqn{
 \Prob[\EB_{i_1}(T) \cap \ED_{i_1, i_2, l, I, S}(T)]
}\nn
& = \sum_{T = \sum_{i=1}^K \sum_{j<i}^K \nij{i}{j}} \int_{\{\NT{i}{j} = \nij{i}{j}\}} \Ind\{\EB_{i_1}(T) \cap \ED_{i_1, i_2, l, I, S}(T)\} \e^{\hKL(\{\nij{i}{j}\}_{(i,j) \in \SetIS})} d\Probp_{i_1, i_2, l, I, S} \nn
& \le T^{1-\lbepsilontwo} \Probp_{i_1, i_2, l, I, S}[\EB_{i_1}(T)] \le o(T^{a-\lbepsilontwo}) = o(1).
\label{ineq_lem_steptwo}
\end{align}

We finally obtain
\begin{align*}
\Prob[\EA(T)] 
&= \Prob\left[ \cap_{i_1 \in \winners} \cup_{i_2}\,\cup_{l}\,\cup_{I}\,\cup_{S}\, \EE_{i_1, i_2, l, I, S}(T) \right] \\
&\le \Prob\left[ \cap_{i_1 \in \winners} \cup_{i_2}\,\cup_{l}\,\cup_{I}\,\cup_{S}\, \{\EE_{i_1, i_2, l, I, S}(T) \cap  \ED_{i_1, i_2, l, I, S}^c(T)\} \right] \\
& \hspace{2em} + \Prob\left[ \cap_{i_1 \in \winners} \cup_{i_2}\,\cup_{l}\,\cup_{I}\,\cup_{S}\,\ED_{i_1, i_2, l, I, S}(T) \right] \\
&= o(1) + \Prob\left[ \cap_{i_1 \in \winners} \cup_{i_2}\,\cup_{l}\,\cup_{I}\,\cup_{S}\,\ED_{i_1, i_2, l, I, S}(T) \right] \text{\phantom{www}(by union bound of \eqref{ineq_lem_stepone} over $i_1, i_2, l, I, S$).}\\
%&\le \Prob\left[ \cup_{i_1 \in \winners} \left\{ \EB_i(T) \cap \left(\cup_{i_2}\,\cup_{l}\,\cup_{I}\,\cup_{S}\, \ED_{i_1, i_2, l, I, S}(T) \right) \right\} \right] + o(1) 
\end{align*}
Remember that $\cup_{i \in \winners} \EB_i(T)$ occurs with probability $1-o(1)$ (inequality \eqref{ineq_ebcopelands}). 
By using
\begin{align*}
\lefteqn{
\left\{
\cap_{i_1 \in \winners} \cup_{i_2}\,\cup_{l}\,\cup_{I}\,\cup_{S}\, \ED_{i_1, i_2, l, I, S}(T)
\right\}
 \cap \bigcup_{i \in \winners} \EB_i(T)
} \\
& \subset \cup_{i_1 \in \winners} \left\{ \EB_i(T) \cap \left(\cup_{i_2}\,\cup_{l}\,\cup_{I}\,\cup_{S}\, \ED_{i_1, i_2, l, I, S}(T) \right) \right\},
\end{align*}
we have
\begin{align*}
\lefteqn{
\Prob\left[ \cap_{i_1 \in \winners} \cup_{i_2}\,\cup_{l}\,\cup_{I}\,\cup_{S}\,\ED_{i_1, i_2, l, I, S}(T) \right]
}\nn
&=\Prob\left[ \cup_{i_1 \in \winners} \left\{ \EB_i(T) \cap \left(\cup_{i_2}\,\cup_{l}\,\cup_{I}\,\cup_{S}\, \ED_{i_1, i_2, l, I, S}(T) \right) \right\} \right] + o(1) \text{\phantom{www}(by \eqref{ineq_ebcopelands})}\nn
&=o(1) \text{\phantom{www}(by union bound of \eqref{ineq_lem_steptwo} over $i_1, i_2, l, I, S$)}.
\end{align*}
In summary, $\Prob[\EA(T)] = o(1)$ and thus the proof is completed.

\end{proof}

\begin{proof}[Proof of Theorem \ref{thm_regretlower}]
Assume that
there exists $\delta > 0$ and 
a sequence
$T_1<T_2<T_3<\cdots$ such that for all $s$
\begin{equation*}
\Expect[\Regret(T_s)]
<
(1-\delta)\min_{i_1 \in \winners} \optcone{i_1}(\{\muij{i}{j}\})\log T_s\com
\end{equation*}
that is, there exists $\SetIS(s)$ such that
\begin{align*}
\sum_{(i,j) \in \SetIS(s)}\frac{\Expect[\myN{i}{j}{T_s}]}{(1-\delta)\log T_s} \reg{i}{j} < \min_{i_1 \in \winners} \optcone{i_1}(\{\muij{i}{j}\})\per
\end{align*}
Let $\is \in \winners$ be arbitrary and $\mathcal{S}$ be the closure of the space of preference matrices in which $i_1$ is not the Copeland winner, that is, $\mathcal{S} = \closure(\{\{\nuij{i}{j}\}_{i>j}: i_1 \notin \winnersf(\{\nuij{i}{j}\})\})$.
From the definition of $\optcone{i_1}$,
there exists $\{\nuij{i}{j}(s)\} \in \mathcal{S}$ such that 
\begin{equation*}
\sum_{(i,j)\in\SetIS(s)} \frac{\Expect[\myN{i}{j}{T_s}]}{(1-\delta)\log T_s} \KL(\muij{i}{j}, \nuij{i}{j}(s)) < 1\per
\end{equation*}
Since $\mathcal{S}$ is compact, there exists a subsequence
$s_0<s_1<\cdots$ such that
$\lim_{u\to\infty} \{\nuij{i}{j}(s_u)\} = \{\nuijd{i}{j}\}$
for some $\{\nuijd{i}{j}\}\in\mathcal{S}$.
Therefore from the lower semicontinuity of the divergence we obtain
\begin{align*}
1
&\ge \liminf_{u\to\infty} \sum_{(i,j)\in\SetIS(s_u)}
\frac{\Expect[\myN{i}{j}{T_{s_u}}]}{(1-\delta)\log T_{s_u}} \KL(\muij{i}{j}, \nuij{i}{j}(s_{u}))\nn
&\ge \liminf_{u\to\infty} \sum_{(i,j)\in\SetIS(s_u)} 
\frac{\Expect[\myN{i}{j}{T_{s_u}}]}{(1-\delta)\log T_{s_u}}
\KL(\muij{i}{j}, \nuijd{i}{j})\com
\end{align*}
which contradicts Lemma \ref{lem_lbnumarm}.
\end{proof}

\section{Proof on an Efficient Computation of ECW-RMED}
\label{sec_effcomp}

\begin{figure*}[t!]
\begin{center}
  \setlength{\subfigwidth}{.49\linewidth}
  \addtolength{\subfigwidth}{-.49\subfigcolsep}
  \begin{minipage}[t]{\subfigwidth}
  \centering
 \subfigure[Case (i)]{\includegraphics[scale=0.5]{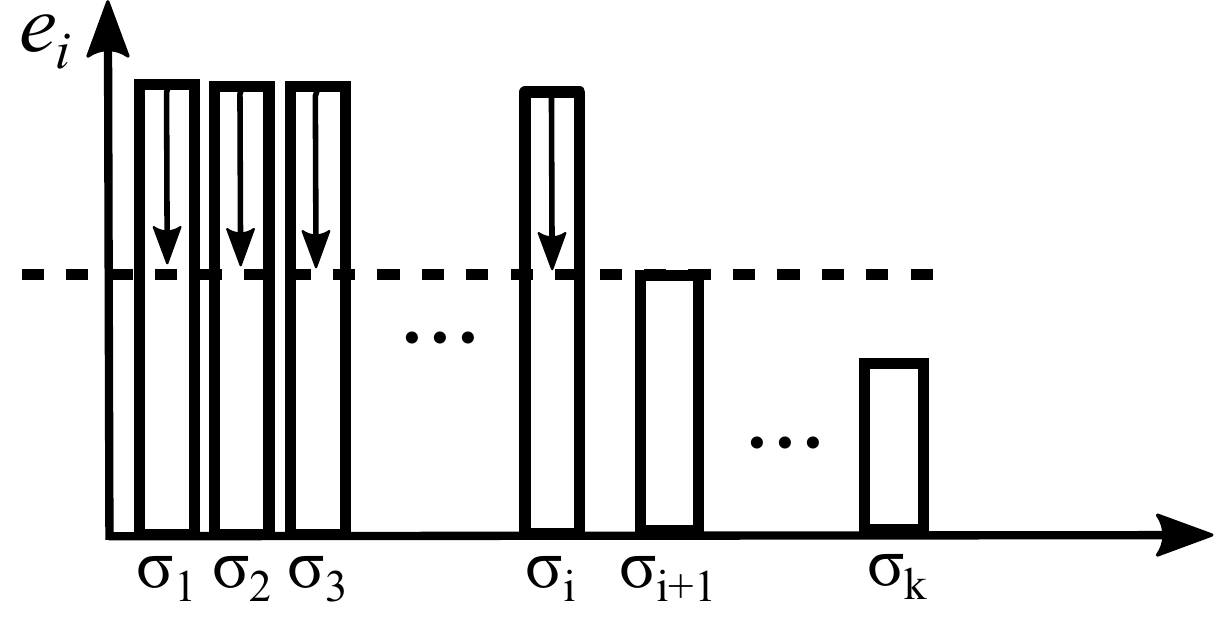}\label{fig_costoptone}
 }
 \vspace{-0.501em}
  \end{minipage}\hfill
  \begin{minipage}[t]{\subfigwidth}
  \centering
 \subfigure[Case (ii)]{\includegraphics[scale=0.5]{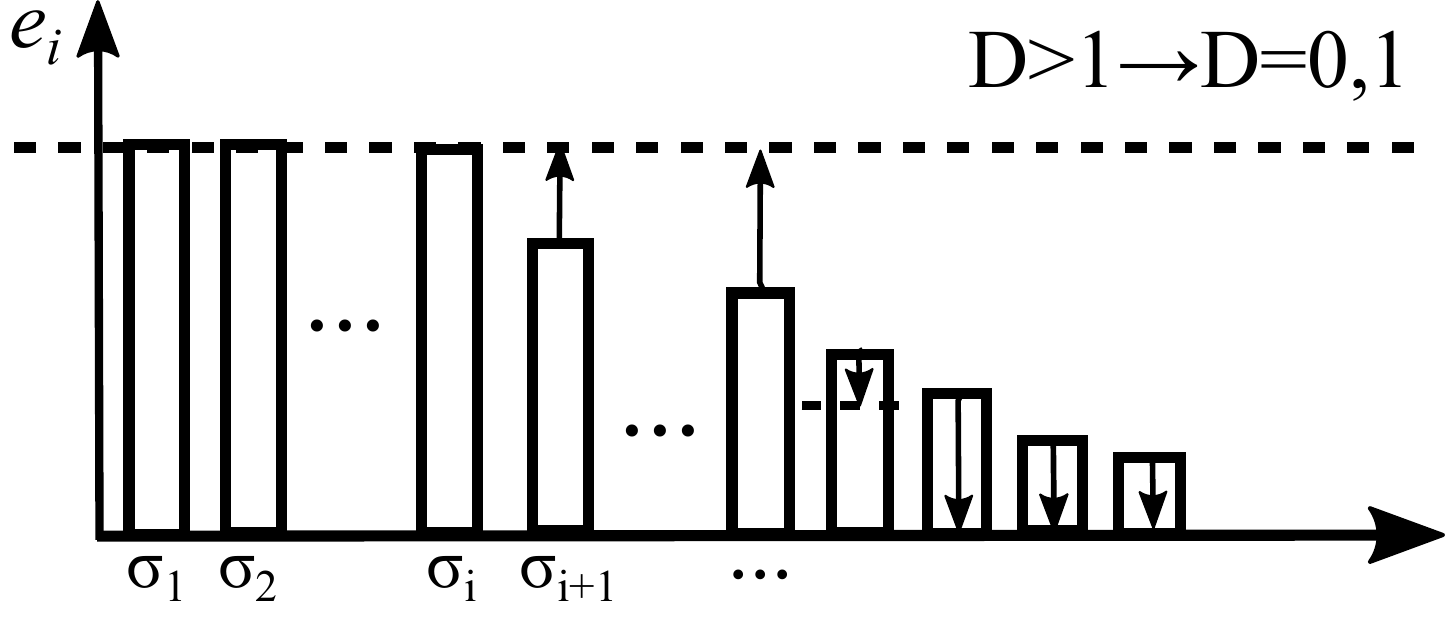}
\label{fig_costopttwo}
 }
 \vspace{-0.501em}
  \end{minipage}
  \begin{minipage}[t]{\subfigwidth}
  \centering
 \subfigure[Case (iii)]{\includegraphics[scale=0.5]{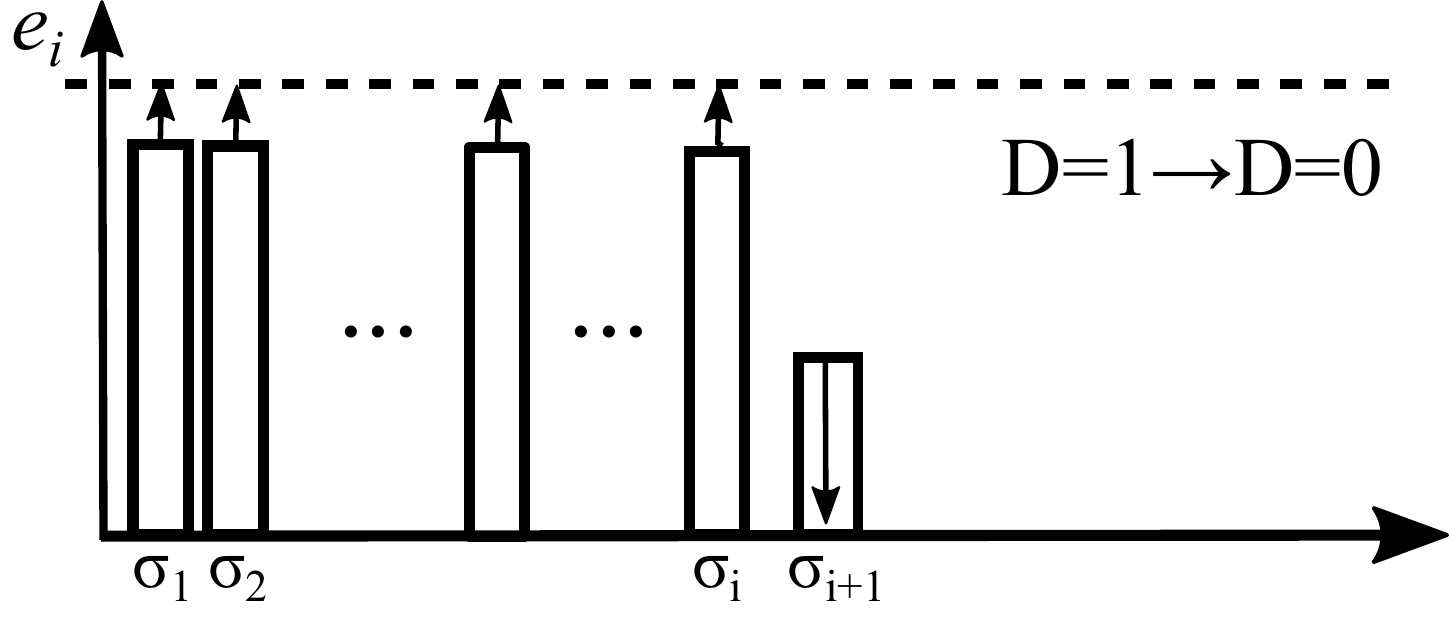}
\label{fig_costoptthree}
 }
 \vspace{-0.501em}
 \end{minipage}\hfill
  \begin{minipage}[t]{\subfigwidth}
  \centering
 \subfigure[An optimal solution with $D=0$.]{
 \includegraphics[scale=0.5]{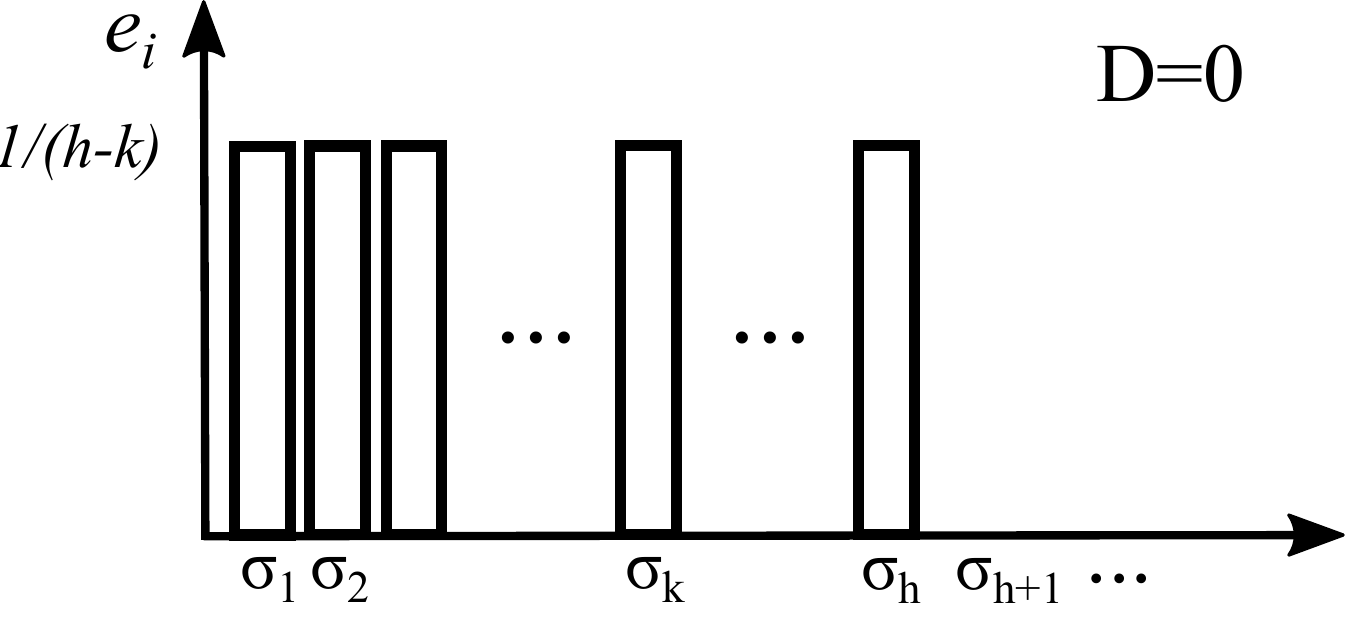}\label{fig_costopt}
 }
 \vspace{-0.501em}
 \end{minipage}\hfill
\end{center}
  \caption{On the solution of the optimization problem of \eqref{ineq_oopt}. Figure (a)--(c) illustrate the operations that convert an optimal solution into another one with a smaller value of $D$. Figure (d) illustrates an optimal solution such that $D=0$.}
 \label{fig:regret1}
\end{figure*}%

\begin{proof}[Proof of Theorem \ref{thm_hsol}]

First, we show that there exists a optimal solution of \eqref{ineq_oopt} such that
\begin{equation}
\normvardash{\sigma_1} \ge \normvardash{\sigma_2} \ge \dots \ge \normvardash{\sigma_{|\mO|}}
\label{normvarorder}
\end{equation}
for the following reason; let $\{\normvardashtwo{j}\}$ be an arbitrary optimal solution. If there exists a pair $i<j$ such that $\normvardashtwo{\sigma_i} > \normvardashtwo{\sigma_j}$, swapping the values of $\normvardashtwo{\sigma_i}$ for $\normvardashtwo{\sigma_j}$ does not increase the objective value since $\cost{\sigma_i} \le \cost{\sigma_j}$, and recursively applying this swap operation yields another optimal solution $\{\normvardash{j}\}$ such that \eqref{normvarorder} holds. 
The constraint in \eqref{ineq_oopt} for $\{\normvardash{j}\}$ satisfying \eqref{normvarorder} is equivalent to 
\begin{equation}
\sum_{o \in k+1,\dots,|\mO|} \normvardash{\sigma_o} \geq 1. \label{newconstraint}
\end{equation}

Let the number of gaps be $D := \sum_{i=1}^{|\mO|-1}\Ind\{\normvardash{\sigma_i}>\normvardash{\sigma_{i+1}},\normvardash{\sigma_{i+1}}>0\}$. In the following, we show that if $D>0$ there exists another optimal solution with a smaller value of $D$.
Let $D>0$ and $i<|\mO|$ be the smallest index such that $\normvardash{\sigma_i} > \normvardash{\sigma_{i+1}} > 0$. 
(i) if $i \le k$, replacing $\normvardash{\sigma_1},\dots,\normvardash{\sigma_i}$ with $\normvardash{\sigma_{i+1}}$ does not increase the objective since each $\cost{\sigma_i}$ is non-negative. This operation yields another optimal solution that satisfies \eqref{normvarorder} and \eqref{newconstraint} with a smaller value of $D$, which is illustrated in Figure \ref{fig_costoptone}.
(ii) If $i > k$ and $D\ge2$, let $S = \sum_{j=i+1}^{|\mO|} \normvardash{\sigma_j}$. Then,
\begin{equation*}
\normvardashtwo{\sigma_j} = 
\begin{cases}
  \normvardash{\sigma_i} & \text{(if $(j-i)\normvardash{\sigma_i} \ge S$)} \\
  \normvardash{\sigma_i}(j-i) - S & \text{(if $(j-i+1)\normvardash{\sigma_i} \ge S > (j-i)\normvardash{\sigma_i} $)} \\
  0 & \text{(otherwise)}
\end{cases}
\end{equation*}
has equal or smaller value of the objective since $\cost{\sigma_j}$ is non-decreasing in $j$. Therefore, $\{\normvardashtwo{j}\}$ is an optimal solution with $D \le 1$ such that \eqref{normvarorder} and \eqref{newconstraint} hold, which is illustrated in Figure \ref{fig_costopttwo}.
(iii) If $i > k$ and $D=1$, then $\sum_{j=1}^i \cost{\sigma_j} = (i-k)\cost{\sigma_{i+1}}$ always hold. Otherwise, for sufficiently small $\delta>0$ either of (iii-a) increasing $\normvardash{\sigma_1},\dots,\normvardash{\sigma_i}$ by $\delta$ and decreasing $\sigma_{i+1}$ by $(i-k)\delta$ or (iii-b) decreasing $\normvardash{\sigma_1},\dots,\normvardash{\sigma_i}$ by $\delta$ and increasing $\sigma_{i+1}$ by $(i-k)\delta$ must decrease the objective, which contradicts the assumption that $\{\normvardash{j}\}$ is optimal. Therefore, $\sum_{j=1}^i \cost{\sigma_j} = (i-k)\cost{\sigma_{i+1}}$. Then, 
\begin{equation*}
\normvardashtwo{\sigma_j} = 
\begin{cases}
  \normvardash{\sigma_i} + \normvardash{\sigma_{i+1}}/(i-k)  & \text{(if $j \le i$)} \\
  0 & \text{(otherwise)}
\end{cases}
\end{equation*}
has the same value of the objective function, and it satisfies \eqref{normvarorder} and \eqref{newconstraint}. Therefore, $\{\normvardashtwo{j}\}$ is an optimal solution with $D=0$, which is illustrated in Figure \ref{fig_costoptthree}. In summary, if $D>0$, one can apply one of the operations (i)--(iii) that yields a modified optimal solution with a smaller value of $D$. Applying these operations yields the desired solution $\{\normvarstar{j}\}$ with $D=0$, which is illustrated in Figure \ref{fig_costopt}.
\end{proof}

\section{Proof of Theorem \ref{thm_relaxation}}
\label{sec_proof_relaxation}

\begin{proof}[Proof of Theorem \ref{thm_relaxation}]

First, one can check that $\eqref{lloses}$ for each $i_2 \ne i_1$ is equivalent to the constraints of \eqref{optcond} for $l = \Li{1}-1$.
Second, Equation \eqref{lwins} implies that the constraints of \eqref{optcond} for all $i_2 \ne i_1, l \ge \Li{1}$.  Combining these two facts, we conclude that $\{\optvar{i}{j}\} \in \admqrelax{i_1}{\{\muij{i}{j}\}}$ implies $\{\optvar{i}{j}\} \in \admq{i_1}{\{\muij{i}{j}\}}$, and thus \eqref{optimality_always} is proven.

Moreover, to derive \eqref{optimality_multi}, it suffices to show $\admq{i_1}{\{\muij{i}{j}\}} = \admqrelax{i_1}{\{\muij{i}{j}\}}$ for any preference matrix $\{\muij{i}{j}\}$ in which two or more Copeland winners exists. In that case, $\Li{1} = \Li{2}$, and $l$ in \eqref{optcond} runs for $\{\Li{1}-1, \Li{1}\}$. One can check that \eqref{lwins} is equivalent to the constraints of \eqref{optcond} for $l = \Li{1}$. Since $\eqref{lloses}$ for each $i_2 \ne i_1$ is equivalent to the constraints of \eqref{optcond} for $l = \Li{1}-1$, the constraints of $\admq{i_1}{\{\muij{i}{j}\}}$ and $\admqrelax{i_1}{\{\muij{i}{j}\}}$ are equivalent.
 
\end{proof}

\section{Proofs of Theorems \ref{thm_opt} and \ref{thm_relax}}
\label{sec_mainproof}

In this section, we provide full proofs of Theorems \ref{thm_opt} and \ref{thm_relax}.
We define the following events that are important in bounding regret.  
Let 
\begin{equation*}
\EXt{i}{j}:= \left\{ \{\Nt{i}{j} < \algalpha \sqrt{\log t}\} \cup \{ |\hatmut{i}{j} - 1/2| < \algbeta/\log{\log{t}}\} \right\}
\end{equation*}
and $\EXdt{i}{j}$ be the event that $\EXt{i}{j}$ and pair $(i,j)$ is drawn.
Let $\EYt{i}{j}$ be the event that pair $(i,j)$ is added into $L_N$. Note that $\cap_{(i',j')\in\SKP}\EXct{i'}{j'}$ implies the algorithm reaches Line 5 in Algorithm \ref{alg_rmedbase}, and $\EYt{i}{j}$ implies $\cap_{(i',j')\in\SKP}\EXct{i'}{j'}$.
Moreover, let
\begin{equation*}
\EZt{\sensedelta}=\cap_{(i,j)\in\SKP}\{|\hatmut{i}{j}-\muij{i}{j}|<\sensedelta\}.
\end{equation*}

In the following, we first show some lemmas, and then bounds the regret. The proofs of the lemmas are in the following sections of this appendix.
\begin{lemma} {\rm (Case that arms are immediately drawn)}
For CW/ECW-RMED, the following inequality holds:
\begin{equation*}
 \sum_{t=1}^T \Prob[\EXdt{i}{j}] \le \so(\log{T}).
\end{equation*}
\label{lem_exploration}
\end{lemma}
\begin{lemma} {\rm (Case that Copeland winner is not properly estimated)}
For CW/ECW-RMED, for any $i_2 \in \nonwinners$ the following inequality holds:
\begin{equation}
 \sum_{t=1}^T \Prob\left[\bigcap_{(i',j')\in\SKP}\EXct{i'}{j'}, \ist = i_2\right] = O(1).
\label{lem_copenw}
\end{equation}
\label{lem_copeest}
\end{lemma}
\begin{lemma} {\rm (The continuity of the optimal solution)}
Let the algorithm be CW-RMED.
For $(i,j)\in\SKP$, let $\optsetij{i}{j}$ be the $ij$-th component of the unique element of $\optset{\is}{\{\muij{i}{j}\}}$ such that $\is = \argmin_{i_1 \in \winners} \optcone{i_1}(\{\muij{i}{j}\})$. There exists $\senseepsilon(\sensedelta)$ such that $\senseepsilon \rightarrow 0$ as $\delta \rightarrow +0$, and for any $(i,j)\in\SKP$,
\begin{align*}
\sum_{t=1}^T\Ind[\EYt{i}{j},\EZt{\sensedelta}]
\le (1+\senseepsilon(\sensedelta)) \optsetij{i}{j} \log{T}+1.
\end{align*}
\label{lem_sense}
Let the algorithm be ECW-RMED.
We can define $\optsetrelaxij{i}{j}$ in the same way and 
\begin{align*}
\sum_{t=1}^T\Ind[\EYt{i}{j},\EZt{\sensedelta}]
\le (1+\senseepsilon(\sensedelta)) \optsetrelaxij{i}{j} \log{T}+1.
\end{align*}
\end{lemma}

\begin{lemma} {\rm (The regret when the solution quality is low)}
For CW/ECW-RMED, the following inequality holds:
\begin{equation*}
\sum_{t=1}^T\Prob\left[\bigcap_{(i',j')\in\SKP}\EXct{i'}{j'},\EYt{i}{j},\EZct{\sensedelta}\right] = o(\log{T}).
\end{equation*}
\label{lem_xyzc}
\end{lemma}

Note that the regret per round satisfies $\reg{i}{j}\le 1$. 
For each pair $(i,j)$ to be drawn, it either (i) satisfies $\EXt{i}{j}$ (only for pair $i\ne j$), (ii) was put into $L_N$ in the previous loop, or (iii) is in the first loop of $L_C$ (only for pair $i\ne j$). 
By using this, the regret is bounded as 
\begin{align}
\Regret(T)
&=\sum_{(i,j)\in\pairs}\reg{i}{j}\sum_{t=1}^T\Ind[p(t)=(i,j)]\nn
&\le \sum_{(i,j)\in\SKP}\sum_{t=1}^T\Ind[\EXdt{i}{j}]
+\sum_{(i,j)\in\pairs}\reg{i}{j}\sum_{t=1}^T \Ind[\cap_{(i',j')\in\SKP}\EXct{i'}{j'},\EYt{i}{j}]+\sum_{(i,j)\in\SKP} 1\nn
&\le \sum_{(i,j)\in\SKP}\sum_{t=1}^T\Ind[\EXdt{i}{j}]+\sum_{(i,i):i\in\nonwinners}\sum_{t=1}^T\Ind[\cap_{(i',j')\in\SKP}\EXct{i'}{j'},\EYt{i}{i}]+\sum_{(i,j)\in\SKP}\reg{i}{j}\sum_{t=1}^T\Ind[\EYt{i}{j},\EZt{\sensedelta}]\nn
&\hspace{2em}+\sum_{(i,j)\in\SKP}\sum_{t=1}^T(\Ind[\cap_{(i',j')\in\SKP}\EXct{i'}{j'},\EYt{i}{j},\EZct{\sensedelta}])+K^2 \label{ineq_mainreg_terms}
\end{align}
In the following, we bound each term in \eqref{ineq_mainreg_terms} in expectation.
First,
\begin{equation}
\sum_{(i,j)\in\SKP}\sum_{t=1}^T\Prob[\EXdt{i}{j}]=\so(\log{T})
\label{ineq_x}
\end{equation}
follows from Lemma \ref{lem_exploration}.
Second,
\begin{align}
\lefteqn{
\sum_{(i,i):i\in\nonwinners}\sum_{t=1}^T\Prob[\cap_{(i',j')\in\SKP}\EXct{i'}{j'},\EYt{i}{i}]
}\nn
&\le\sum_{(i,i):i\in\nonwinners}\sum_{t=1}^T\Prob[\cap_{(i',j')\in\SKP}\EXct{i'}{j'},\ist = i] \text{\phantom{www} (by the fact that $\EYt{i}{i}$ implies $\ist = i$)}\nn
&=\lo(1) \text{\phantom{wwwwwwwwwwwwwwwwwwwwwwwwwwww} (by the union bound of Lemma \ref{lem_copeest} over $\nonwinners$)}.\nn
\label{ineq_xy}
\end{align}
Third, let the algorithm be CW-RMED. From Lemma \ref{lem_sense},
\begin{align}
\sum_{(i,j)\in\SKP}\reg{i}{j} \sum_{t=1}^T \Ind[\EYt{i}{j},\EZt{\sensedelta}]
&\le (1+\senseepsilon(\sensedelta)) \sum_{(i,j)\in\SKP}\reg{i}{j} \left( \optsetij{i}{j}\log{T}+1\right) \nn
&\le (1+\senseepsilon(\sensedelta)) (\min_{i_1} \optcone{i_1}(\{\muij{i}{j}\})\log{T}+K^2
\label{ineq_xyz}
\end{align}
The same arguments yields the following bound for ECW-RMED:
\begin{align*}
\sum_{(i,j)\in\SKP}\reg{i}{j} \sum_{t=1}^T \Ind[\EYt{i}{j},\EZt{\sensedelta}]
\le (1+\senseepsilon(\sensedelta)) (\min_{i_1} \optconerelax	{i_1}(\{\muij{i}{j}\})\log{T}+K^2.
\end{align*}
Finally, 
\begin{equation}
\sum_{(i,j)\in\SKP} \sum_{t=1}^T\Prob\left[\bigcap_{(i',j')\in\SKP}\EXct{i'}{j'},\EYt{i}{j},\EZct{\sensedelta}\right] = o(\log{T})
\label{ineq_xyzc}
\end{equation}
follows from Lemma \ref{lem_xyzc}.

Combining \eqref{ineq_mainreg_terms}, \eqref{ineq_x}, \eqref{ineq_xy}, \eqref{ineq_xyz}, \eqref{ineq_xyzc} completes the proof.

\section{Proof of Lemma \ref{lem_exploration}}

\begin{proof}[Proof of Lemma \ref{lem_exploration}]
We have,
\begin{align}
\sum_{t=1}^T\Prob[\EXdt{i}{j}]
&= \sum_{n=1}^T\Prob\left[\bigcup_{t=n}^T\left\{\Nt{i}{j}<\algalpha \sqrt{\log{t}} \cup |\hatmut{i}{j} - 1/2| < \algbeta/\log{\log{t}},\Nt{i}{j}=n\right\}\right] \nn
&\le \alpha \sqrt{\log{T}} + \sum_{n=1}^T\Prob\left[\bigcup_{t=n}^T\left\{|\hatmut{i}{j} - 1/2| < \algbeta/\log{\log{t}},\Nt{i}{j}\ge\alpha\sqrt{\log{T}},\Nt{i}{j}=n\right\}\right]. \label{ineq_leone}
\end{align}
Let $\lllT = \log{\log{(\alpha \sqrt{\log{T}})}}$. By using
\begin{equation*}
|\hatmut{i}{j} - 1/2| \ge |\muij{i}{j}-1/2| - |\hatmut{i}{j} - \muij{i}{j}| 
\end{equation*}
we have
\begin{align}
\lefteqn{
\sum_{n=1}^T\Prob\left[\bigcup_{t=n}^T\left\{|\hatmut{i}{j} - 1/2| < \algbeta/\log{\log{t}},\Nt{i}{j}\ge\alpha\sqrt{\log{T}},\Nt{i}{j}=n\right\}\right]
}\nn
&\le\sum_{n=1}^T\Prob[|\hatmun{i}{j}{n} - \muij{i}{j}| > \algbeta/\lllT]
+ \sum_{n=1}^T\Prob[|\muij{i}{j}-1/2| < (2\algbeta)/(\log\log n)]\nn
&\le\sum_{n=1}^T\Prob[|\hatmun{i}{j}{n} - \muij{i}{j}| > \algbeta/\lllT]
+ e^{e^{2\beta/|\muij{i}{j}-1/2|}}\nn
&\le 2\sum_{n=1}^\infty e^{-2n(\algbeta/\lllT)^2}
+ e^{e^{\beta/(2|\muij{i}{j}-1/2|)}} \text{\,(by Chernoff bound and Pinsker's inequality)}\nn
&\le O\left(\frac{\lllT^2}{\algbeta^2}\right) + e^{e^{\beta/(2|\muij{i}{j}-1/2|)}} = o(\log{T}).\label{ineq_lethree}
\end{align}
Combining \eqref{ineq_leone} and \eqref{ineq_lethree} completes the proof.
\end{proof}

\section{Proof of Lemma \ref{lem_copeest}}

\begin{proof}[Proof of Lemma \ref{lem_copeest}]
Note that we can assume $\hatmut{i}{j} \ne 1/2$ from $\EXct{i}{j}$.
Let $i_1 \in \winners$ be arbitrary. Event $\{\ist = i_2\}$ implies that, there exists a set of pairs $\SetIS$ such that $l \in \{\max\{0,\Lone-1\},\dots,\Li{i_2}\}$, $I \in \SubPowSetH{i_1}{l+1-\Lone}$, $S \in \SubPowSetORem{i_2}{\max\{0,\Li{i_2}-l-\Ind\{i_1 \in I\}\}}{i_1}$ and $\SetIS = \{(i_1,j) : j \in I\} \cup \{(i_2,j) : j \in S\}$ and the signs of $\hatmut{i}{j}-1/2$ and $\muij{i}{j}-1/2$ are different. In other words,
\begin{align*}
 \{\ist = i_2\} \subset \{\ist = i_2 \} 
\cap \left\{ \cup_l \cup_I \cup_S \cap_{(i,j) \in \SetIS} \{(\hatmut{i}{j} - 1/2)(\muij{i}{j} - 1/2) < 0\} \right\}.
\end{align*}
In the following we are going to show 
\begin{equation}
 \sum_{t=1}^T \Prob\Bigl[\cap_{(i',j')\in\SKP}\EXct{i'}{j'}, \ist = i_2, \cap_{(i,j) \in \SetIS} \{(\hatmut{i}{j} - 1/2)(\muij{i}{j} - 1/2) < 0\} \Bigr] = O(1)
\label{ineq_hoflip}
\end{equation}
for each $l,I,S$.
Note that 
\begin{equation*}
 \left\{\log{t} \ge \sum_{(i,j)\in\SetIS} \Nt{i}{j}\KL(\hatmut{i}{j}, 1/2), \cap_{(i',j')\in\SKP}\EXct{i'}{j'}, \ist = i_2 \right\}
\end{equation*}
implies that $\{\Nt{i}{j}/\log t\} \notin \admq{i_2}{\{\hatmut{i}{j}\}}$ and at least one of the pairs in $\SetIS$ is immediately put into $L_{NC}$ to satisfy the constraints. Therefore, one of the arms in $\SetIS$ is drawn within $K^2$ rounds of $\{t: \cap_{(i',j')\in\SKP}\EXct{i'}{j'}\}$. 
By using this fact, we have
\begin{multline*}
\sum_t\Ind\left[\cap_{(i',j')\in\SKP}\EXct{i'}{j'}, \ist = i_2, \cap_{(i,j)\in\SetIS}\{(\muij{i}{j}-1/2)(\hatmut{i}{j}-1/2)<0, \Nt{i}{j}=\nij{i}{j}\} \right]\nn
\le\exp{\left(\sum_{(i,j)\in\SetIS} \nij{i}{j}\KL(\hatmut{i}{j}, 1/2)\right)} + K^2.
\end{multline*}
Let $\hatmun{i}{j}{n}$ be the empirical estimate of $\muij{i}{j}$ with $n$ draws. 
Letting 
$\Pij{i}{j}(\xij{i}{j})=\Prob[ (\muij{i}{j}-1/2)(\hatmun{i}{j}{\nij{i}{j}}-1/2) \le 0, \KL(\hatmun{i}{j}{\nij{i}{j}}, 1/2) \ge \xij{i}{j}]$, 
we have 
\begin{align}
\lefteqn{
\Expect\left[\sum_{t} \Ind\left[
\bigcap_{(i',j')\in\SKP}\EXct{i'}{j'}, \ist = i_2, \bigcap_{(i,j) \in \SetIS}\{ (\muij{i}{j}-1/2)(\hatmun{i}{j}{\nij{i}{j}}-1/2)<0,\,\Nt{i}{j} = \nij{i}{j}\}
\right]
\right]
}\nn
&\le
\int_{\{\xij{i}{j}\}\in [0,\log 2]^{|\SetIS|}}
\left(
\exp\left(\sum_{(i,j) \in \SetIS}\nij{i}{j} \xij{i}{j} 
\right)
+K^2
\right)
\prod_{(i,j) \in \SetIS}\rd (-\Pij{i}{j}(\xij{i}{j}))\nn
&=
K^2 \prod_{(i,j) \in \SetIS}\Pij{i}{j}(0)+
\prod_{(i,j) \in \SetIS}\int_{\xij{i}{j}\in[0,\log 2]}
\e^{\nij{i}{j} \xij{i}{j}}
\rd (-\Pij{i}{j}(\xij{i}{j}))
\nn
&=
K^2\prod_{(i,j) \in \SetIS}\Pij{i}{j}(0)+
\prod_{(i,j) \in \SetIS}
\left(
\left[-\e^{\nij{i}{j} \xij{i}{j}}\Pij{i}{j}(\xij{i}{j})\right]_0^{\log 2}
+
\int_{\xij{i}{j}\in[0,\log 2]}
\nij{i}{j}\e^{\nij{i}{j} \xij{i}{j}}
\Pij{i}{j}(\xij{i}{j})\rd \xij{i}{j}
\right)
\nn &
\qquad(\mbox{integration by parts})\nn
&\le
(1+K^2)\prod_{(i,j) \in \SetIS}\Pij{i}{j}(0)+
\prod_{(i,j) \in \SetIS}
\int_{\xij{i}{j}\in[0,\log 2]}
\nij{i}{j}\e^{\nij{i}{j} \xij{i}{j}}
\e^{-\nij{i}{j} (\xij{i}{j}+C_1(\muij{i}{j}, 1/2))}\rd \xij{i}{j}
\nn
& \qquad(\mbox{by Chernoff bound and Fact \ref{fact:minimumdivergencediff}, where $C_1(\mu, \mu_2) = (\mu - \mu_2)^2 / (2 \mu (1- \mu_2))$})\nn
&\le
(1+K^2)\prod_{(i,j) \in \SetIS}\e^{-\nij{i}{j} \KL(1/2,\muij{i}{j})}+
\prod_{(i,j) \in \SetIS}
\int_{\xij{i}{j}\in[0,\log 2]}
\nij{i}{j}\e^{-\nij{i}{j} C_1(\muij{i}{j}, 1/2)}\rd \xij{i}{j}
\nn
&=
(1+K^2)\prod_{(i,j) \in \SetIS}\e^{-\nij{i}{j} \KL(1/2,\muij{i}{j})}+
\prod_{(i,j) \in \SetIS}
(\log 2)\nij{i}{j}
\e^{-\nij{i}{j} C_1(\muij{i}{j}, 1/2)}
. \label{ineq:uprobintegral}
\end{align}
By summing \eqref{ineq:uprobintegral} over $\{\nij{i}{j}\}$,
\begin{align*}
\lefteqn{
\sum_{t=1}^T
\Prob\left[
\bigcap_{(i',j')\in\SKP}\EXct{i'}{j'}, \bigcap_{(i,j) \in \SetIS}\{ (\muij{i}{j}-1/2)(\hatmun{i}{j}{\nij{i}{j}}-1/2)<0\}
\right]
}\nn
&\le \sum\dots\sum_{\hspace{-3em}\{\nij{i}{j}\} \in \Natural^{|\SetIS|}}\left(
(1+K^2)\prod_{(i,j) \in \SetIS}\e^{-\nij{i}{j} \KL(1/2,\muij{i}{j})}+
\prod_{(i,j) \in \SetIS}
(\log 2)\nij{i}{j}
\e^{-\nij{i}{j} C_1(\muij{i}{j}, 1/2)}
\right)\nn
&\le 
 (1+K^2)\prod_{(i,j) \in \SetIS}\frac{1}{\e^{\KL(1/2,\muij{i}{j})}-1}
+(\log{2})^{|\SetIS|} \prod_{(i,j) \in \SetIS}\frac{\e^{C_1(\muij{i}{j}, 1/2)}}{(\e^{C_1(\muij{i}{j}, 1/2)}-1)^2}
,\nn
&=\lo(1)
\end{align*}
where we used the fact that $\sum_{n=1}^\infty\e^{-nx}=1/(\e^x+1)$ and $\sum_{n=1}^\infty n\e^{-nx}=\e^x/(\e^x+1)^2$. In summary, we showed \eqref{ineq_hoflip}.
Taking a union bound over $l,I,S$ yields \eqref{lem_copenw}.
\end{proof}

\section{Proof of Lemma \ref{lem_sense}}

Following \citet{hogan}, we define the continuity of a point-to-set map $\Omega: X \rightarrow 2^Y$ between metric spaces $X$ and $Y$ as follows: (i) $\Omega$ is open at $x_0 \in X$ if $\{x^k\}$, $x^k \rightarrow x_0$, and $y_0 \in \Omega(x_0)$ imply the existence of an integer $m$ and a sequence $\{y^k\}$ such that $y^k \in \Omega(x^k)$ for $k\ge m$ and $y^k \rightarrow y_0$. (ii) $\Omega$ is closed at $x_0$ if $\{x^k\} \in X$, $x^k \rightarrow x_0$, $y^k \rightarrow y_0$ imply that $y_0 \in \Omega(x_0)$. Moreover, (iii) $\Omega$ is continuous at $x_0$ if it is closed and open at $x_0$.

Let a set of relaxed feasible solutions be
\begin{align*}
\lefteqn{
\admqzero{i_1}{\{\nuij{i}{j}\}} := \Biggl\{ \{\optvar{i}{j}\}_{i>j} \in [0,1/\KL(\nuij{i}{j}, 1/2)\cred{+1}]^{K(K-1)/2} : 
 \forall_{i_2 \neq i_1}\,\forall{l \in \{\max\{0,\Lifo{1}-1\},\dots,\Lifo{2}\} } 
} \nn
&\hspace{8em}\forall{I \in \SubPowSetHf{i_1}{(l+1-\Lifo{1})}}\,\forall{S \in \SubPowSetORemf{i_2}{\max\{0, \Lif{i_2}-l-\Ind\{i_2 \in I\}\}}{i_1}} \sum_{(i,j) \in \SetIS} \optvar{i}{j} \KL(\nuij{i}{j}, 1/2) \geq 1
 \Biggr\}.
\end{align*}
Note that the red term is the difference from $\admq{i_1}{\cdot}$. This set of relaxed feasible solutions is introduced for the sake of inequality \eqref{sakeplusone} that appears later. The optimal coefficient $\optconezero{i_1}(\{\nuij{i}{j}\})$ and  the set of the optimal solutions $\optsetzero{i_1}{\{\nuij{i}{j}\}}$ are defined in accordance with $\admqzero{i_1}{\{\nuij{i}{j}\}}$, that is,
\begin{equation*}
 \optconezero{i_1}(\{\nuij{i}{j}\}) := \inf_{\{\optvar{i}{j}\}_{i>j}\in\admqzero{i_1}{\{\nuij{i}{j}\}}} \sum_{(i,j) \in \SKP} \regf{i}{j} \optvar{i}{j} \com
\end{equation*}
and
\begin{align*}
 \optsetzero{i_1}{\{\nuij{i}{j}\}} :=\biggl\{ \{\optvar{i}{j}\}_{i>j}  \in \admqzero{i_1}{\{\nuij{i}{j}\}} : \sum_{(i,j) \in \SKP} \regf{i}{j} \optvar{i}{j} = \optconezero{i_1}(\{\nuij{i}{j}\})
\biggr\}\per %\label{def_optset}
\end{align*}

Let the norms on $\{\nuij{i}{j}\}$ and $\{\optvar{i}{j}\}$ be $|\{\nuij{i}{j}\}| = \sum_{i,j} |\nuij{i}{j}|$ and $|\{\optvar{i}{j}\}|=\sum_{i,j} |\optvar{i}{j}|$, respectively. In the following, we show the following lemma:
\begin{lemma} {\rm (The continuity of the solution function)}
 The point-to-set map $\optsetzero{i_1}{\{\nuij{i}{j}\}}: \mcop \rightarrow 2^{[0,\infty)^{K(K-1)}}$ is continuous at $\{\nuij{i}{j}\} = \{\muij{i}{j}\}$.
\label{lem_cont_inner}
\end{lemma}

The continuity and the uniqueness of the optimal solution function $\optsetzero{i_1}{\{\muij{i}{j}\}}$ implies that all solutions of $\optsetzero{i_1}{\{\nuij{i}{j}\}}$ approach $\optsetzero{i_1}{\{\muij{i}{j}\}}$ ($= \optset{i_1}{\{\muij{i}{j}\}}$, unique) when $\{\nuij{i}{j}\}$ is sufficiently close to $\{\muij{i}{j}\}$. 
To prove Lemma \ref{lem_cont_inner}, we first restate the following three Lemmas of \citet{hogan}:
\begin{lemma} {\rm (Theorem 10 of \citealt{hogan})}
Let $g$ be a set of real-valued functions on $X\times Y$, and $P(x) := \{y \in Y: g(x,y) \le 0\}$ be a map of feasible solutions. If each component of $g$ is continuous on $x_0 \times Y$, then $P$ is closed at $x_0$.
\label{lem_close}
\end{lemma}%
\begin{lemma} {\rm (Theorem 12 of \citealt{hogan})}
If $Y$ is convex and normed, if each component of $g$ is continuous on $x_0 \times P(x_0)$ and convex in $y$ for each fixed $x \in X$, and if there exists a $y_0$ such that $g(x_0, y_0)<0$, then $P$ is open at $x_0$.
\label{lem_open}
\end{lemma}%
\begin{lemma} {\rm (Corollary 8.1 of \citealt{hogan})}
Let $\Omega: X \rightarrow 2^Y$ be a point-to-set map and $M(x):=\{y \in \Omega(x) : \sup_{y' \in \Omega(x)} f(x,y') = f(x,y)\}$ be an optimal solution function of some real-valued function $f$ on $X \times Y$. Suppose $\Omega$ is continuous at $x_0$, $f$ is continuous on $x_0 \times \Omega(x_0)$, $M$ is non-empty and uniformly compact near $x_0$, and $M(x_0)$ is unique. Then, $M$ is continuous at $x_0$.
\label{lem_optcont}
\end{lemma}

\begin{proof}[Proof of Lemma \ref{lem_cont_inner}]
We first show the continuity of the feasible solution function $\admqzero{i}{\{\nuij{i}{j}\}}$ at $\{\nuij{i}{j}\}=\{\muij{i}{j}\}$. 
The continuity of each component of $g$ as a function of $\{\nuij{i}{j}\}, \{\optvar{i}{j}\}$ follows from the continuity of the KL divergence,
and thus, applying Lemma \ref{lem_close} for $P = \admqs{i_1}$, $x_0=\{\muij{i}{j}\}$ and $g(\{\nuij{i}{j}\},\{\optvar{i}{j}\}) = \{1-\sum_{(i,j) \in \SetIS} \optvar{i}{j} \KL(\nuij{i}{j}, 1/2)\}_{i_2,l,I,S}$ yields the closedness of $\admqs{i_1}$ at $\{\muij{i}{j}\}$.
Moreover, by (i) continuity of each component of $g$, (ii) linearity of each component of $g$ as a function of $\{\optvar{i}{j}\}$ for each $\{\nuij{i}{j}\}$, and (iii) the fact that $\{\optvard{i}{j}\} := \{(1/\KL(\nuij{i}{j},1/2))+1\}^{K(K-1)}$ satisfies
\begin{equation}
\sum_{(i,j) \in \SetIS} \optvard{i}{j} \KL(\nuij{i}{j}, 1/2) > 1,
\label{sakeplusone}
\end{equation}
applying Lemma \ref{lem_open} to the same $P,x_0,g$ and $y_0=\{(1/\KL(\nuij{i}{j},1/2))+1\}$ yields the openness of $\admqs{i_1}$ at $\{\muij{i}{j}\}$.
The continuity of $\admqs{i_1}$ follows from its closedness and the openness. 

Finally, by using the continuity of $\admqs{i_1}$ and $\optconezero{i}$, and uniform compactness and uniqueness of $\optsets{i_1}$ at $\{\muij{i}{j}\}$, applying Lemma \ref{lem_optcont} to $M=\optsets{i_1}$, $\Omega=\admqs{i_1}$, and $f = \optcone{i_1}$ yields the continuity of $\optsets{i_1}$ at $\{\muij{i}{j}\}$.
\end{proof}

\begin{proof}[Proof of Lemma \ref{lem_sense}]
By using the continuity of $\optsetzero{i_1}{\{\nuij{i}{j}\}}$ (Lemma \ref{lem_cont_inner}), $\optset{i_1}{\{\nuij{i}{j}\}} \subset\optsetzero{i_1}{\{\nuij{i}{j}\}}$, and the uniqueness of $\argmin_{i_1 \in \winners} \optcone{i_1}(\{\muij{i}{j}\})$ and $\optset{i_1}{\{\muij{i}{j}\}}$, there exists $\senseepsilon(\sensedelta)$ such that $\senseepsilon \rightarrow 0$ as $\delta \rightarrow +0$ and
\begin{align*}
\sum_{t=1}^T\Ind[\EYt{i}{j},\EZt{\sensedelta}]
&\le \sum_{n=1}^T\Ind\left[\bigcup_{t=1}^T \left\{\EYt{i}{j},\EZt{\sensedelta}, \Nt{i}{j}=n\right\}\right]\nn
&\le \sum_{n=1}^T\Ind\left[\bigcup_	{t=1}^T \left\{n/\log{t} \le (1+\senseepsilon(\sensedelta)) \optsetij{i}{j} \right\}\right]\nn
&\le (1+\senseepsilon(\sensedelta)) \optsetij{i}{j}\log{T}+1.
\end{align*}
The same arguments also applies to ECW-RMED.
\end{proof}

\section{Proof of Lemma \ref{lem_xyzc}}

\begin{proof}[Proof of Lemma \ref{lem_xyzc}]
We have
\begin{align}
\lefteqn{
\sum_{t=1}^T\Prob\left[\bigcap_{(i',j')\in\SKP}\EXct{i'}{j'},\EYt{i}{j},\EZct{\sensedelta}\right]
}\nn
&\le\sum_{t=1}^T\Prob\left[\bigcap_{(i',j')\in\SKP}\left\{\EXct{i'}{j'}, \Nt{i'}{j'}\ge(\log{\log T})^{1/3}\right\},\EYt{i}{j},\EZct{\sensedelta}\right]\nn
&+\sum_{(i',j')\in\SKP}\sum_{t=1}^T\Prob[\Nt{i'}{j'}\le(\log{\log T})^{1/3},\Nt{i'}{j'}\ge\myalpha\sqrt{\log{t}}].\label{ineq_xyzc_one}
\end{align}
Here,
\begin{align}
\sum_{(i',j')\in\SKP}\sum_{t=1}^T\Prob[\Nt{i'}{j'}\le(\log{\log T})^{1/3},\Nt{i'}{j'}\ge\myalpha\sqrt{\log{t}}]
\le K^2 e^{\alpha^{-2} (\log{\log T})^{2/3}} = o(\log{T}).\label{ineq_xyzc_two}
\end{align}
Moreover,
\begin{align}
\lefteqn{
\sum_{t=1}^T\Prob\left[\bigcap_{(i',j')\in\SKP}\left\{\EXct{i'}{j'}, \Nt{i'}{j'}\ge(\log{\log T})^{1/3}\right\},\EYt{i}{j},\EZct{\sensedelta}\right]
}\nn
&\le\sum_{n=1}^T\Prob\left[\bigcup_{t=n}^T \left\{|\hatmut{i}{j} - 1/2| \ge \algbeta/\log{\log{t}},\EYt{i}{j},\EZct{\sensedelta},\bigcap_{(i',j')\in\SKP}\Nt{i'}{j'}\ge(\log{\log T})^{1/3},\Nt{i}{j}=n\right\}\right]\nn
&\le\sum_{n=1}^{\log{T}((\log{\log{T}}/\beta)^2/2)}\Prob\left[\bigcup_{t=1}^T \left\{\EZct{\sensedelta},\bigcap_{(i',j')\in\SKP}\Nt{i'}{j'}\ge(\log{\log T})^{1/3},\Nt{i}{j}=n\right\}\right]\nn&\hspace{5em}\text{\,\,\,(by $\admq{i_1}{\{\nuij{i}{j}\}}$,$\admqrelax{i_1}{\{\nuij{i}{j}\}} \subset [0,1/\KL(\nuij{i}{j}, 1/2)]^{K(K-1)/2}$ and Pinsker's inequality)}\nn
&\le e^{-\Omega((\log{\log{T}})^{1/3})} O( (\log{T}) (\log{\log{T}})^2) = \so(\log{T}),\label{ineq_xyzc_three}
\end{align}
where we used the fact that
\begin{align*}
\Prob\left[\bigcup_{t=1}^T \left\{\EZct{\sensedelta},\bigcap_{(i',j')\in\SKP}\Nt{i'}{j'}\ge(\log{\log T})^{1/3}\right\}\right]
&\le\sum_{(i',j')\in\SKP}\sum_{n=(\log{\log T})^{1/3}}^T\Prob[|\hatmun{i}{j}{n}-\muij{i}{j}|>\sensedelta]\nn
&\le\sum_{(i',j')\in\SKP}\sum_{n=(\log{\log T})^{1/3}}^T 2\e^{-2n\delta} = \e^{-\Omega((\log{\log T})^{1/3})}.
\end{align*}
Combining \eqref{ineq_xyzc_one}, \eqref{ineq_xyzc_two}, and \eqref{ineq_xyzc_three} completes the proof.
\end{proof}

\end{document}